\documentclass{article}

\usepackage[main, final]{neurips_2026}    % Main Track (camera-ready)
\usepackage{amsmath, amssymb, amsthm}
\usepackage[utf8]{inputenc}
\usepackage[T1]{fontenc}
\usepackage{comment}
\usepackage{algorithm}
\usepackage{algpseudocode}
\usepackage{float}
\usepackage{xcolor}
\usepackage{graphicx}
\usepackage{tikz}
\usepackage{wrapfig}
\usepackage{booktabs}
\usepackage{bbm}
\usepackage{nicefrac}
\usepackage{microtype}
\usepackage{url}

\PassOptionsToPackage{numbers,compress}{natbib}

\usepackage[pagebackref=true]{hyperref}
\newtheorem{theorem}{Theorem}
\newtheorem{proposition}{Proposition}

\newtheorem{corollary}{Corollary}
\newtheorem{remark}{Remark}

\renewcommand{\geq}{\geqslant}
\renewcommand{\leq}{\leqslant}

\newcommand{\EE}{\mathbb{E}}
\newcommand{\PP}{\mathbb{P}}

\newif\ifarxiv
\arxivfalse

\title{The Limits of AI-Driven Allocation:\\
  Optimal Screening under Aleatoric Uncertainty}

\author{%
  Santiago Cortes-Gomez\thanks{Equal contribution.} \\
  Machine Learning Department \\
  Carnegie Mellon University \\
  \texttt{scortesg@cs.cmu.edu}
  \And
  Mateo Dulce Rubio\footnotemark[1] \\
  Center for Data Science \\
  New York University \\
  \texttt{mateo.d@nyu.edu}
  \And
  Carlos Patiño \\
  University of Amsterdam\\
  \texttt{carlos.patino.paz@student.uva.nl}
  \And
  Bryan Wilder \\
  Machine Learning Department \\
  Carnegie Mellon University \\
  \texttt{bwilder@cs.cmu.edu}
}

\begin{document}
\maketitle

\begin{abstract}
    The rise of machine learning has shifted targeted resource allocation in policy and humanitarian settings toward algorithmic targeting based on predicted risk scores. This approach is typically cheaper and faster than traditional screening procedures that directly observe the latent vulnerability status through physical verification. Yet, even access to the true conditional vulnerability probability cannot eliminate misallocation: aleatoric uncertainty over individual vulnerability status is irreducible, and probabilistic targeting inevitably misallocates some resources. In this work, we study how screening and algorithmic targeting should be optimally combined in a two-stage allocation framework, in which a screening stage observes true outcomes for a subset of units before a final allocation stage assigns the resource under a fixed coverage budget. We show that the optimal strategy screens units \textit{at the margin of algorithmic allocation}, while directly targeting the highest-risk units. We empirically characterize when screening and algorithmic targeting act as complements or substitutes: efficiency gains from screening grow as the aleatoric uncertainty in the population increases. Finally, we characterize the relative value of screening and direct allocation, which yields a simple rule for splitting a program budget between the two: the optimal cutoff depends only on their cost ratio. We illustrate our framework with applications in income-based social protection programs and humanitarian demining in Colombia, where the tension between screening costs and allocation efficiency is operationally consequential. \looseness=-1
\end{abstract}
% We derive the optimal screening set in closed form for uniform risk and provide a fixed-point algorithm that recovers it for arbitrary risk distributions. 

\section{Introduction}

Resource allocation in high-stakes humanitarian and social settings requires distributing a scarce resource---cash transfers, medical treatment, humanitarian aid---across a population with heterogeneous and unobserved needs. In principle, every individual with a latent vulnerability condition should receive the resource, but universal screening for direct verification of eligibility is often prohibitively expensive. For instance, convening medical panels, conducting household visits, or performing field surveys \cite{Harutyunyan2023OperationalEfficiency, who2020screening, hill2010examination} all involve large logistical demands that make universal screening infeasible under realistic budgets.

The rise of machine learning has offered a practical alternative to physical verification. Risk scores estimated from observable covariates can approximate individual vulnerability at scale: a model trained on demographic, geographic, or contextual features assigns each unit a predicted probability of being in need, and decision-makers can allocate resources to the highest-scoring units without any direct outcome measurement. This \textit{pre-hoc} algorithmic targeting circumvents the need to measure vulnerability directly, making it substantially cheaper and faster to deploy than traditional screening strategies. As predictive models have grown more accurate and data more abundant, AI-driven allocation has become the default strategy in many humanitarian and social policy contexts, from poverty targeting via proxy means tests \cite{grosh1995proxy, aiken2022machine} to risk prioritization in humanitarian demining \cite{dulce2024reland, cirillo2024desk}.

Yet, probabilistic targeting cannot fully substitute for screening, even when the true conditional vulnerability probability is available. Aleatoric uncertainty over individual vulnerability remains irreducible as there typically is inherent variation in the actual vulnerability status of units with the same observable features. Some misallocation is therefore unavoidable under any \textit{ex-ante} targeting strategy, even under a perfectly specified model. Screening, by contrast, resolves this uncertainty directly by observing the true status of screened individuals, allowing certain allocation for those units. Whether the efficiency gain from screening justifies its cost depends on how much aleatoric uncertainty the population exhibits and where in the risk distribution that uncertainty is concentrated.

This tension motivates a natural two-stage design that we formalize and study in this work. In the first stage, a decision-maker screens a subset of units directly observing their true vulnerability status. In the second stage, allocation targets vulnerable screened units and remaining resources are allocated using the vulnerability risk model for the unscreened units. The central question is then: \textit{which units should be screened to maximize allocation efficiency?} Notably, the optimal screening policy does not target the highest- or lowest-risk units, but those \textit{at the margin of algorithmic allocation} (Figure \ref{fig:screening_policy}). This result characterizes when screening and algorithmic targeting act as complements or substitutes, with efficiency gains from screening growing as population-level aleatoric uncertainty increases.

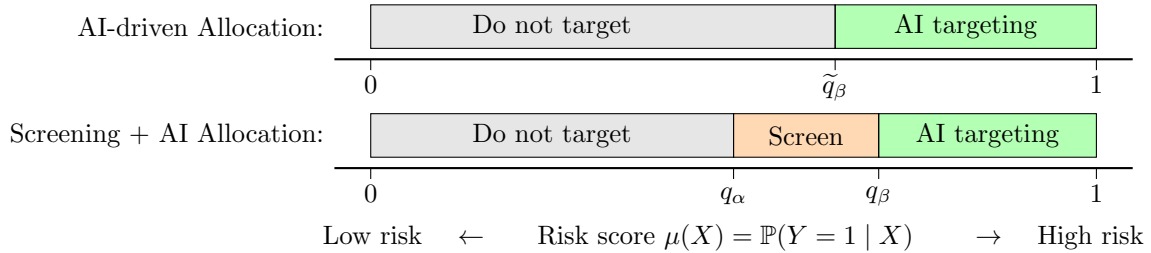
\begin{figure}[h]
\centering
\begin{tikzpicture}[scale=0.96]

  % ── TOP DIAGRAM ──────────────────────────────────────────────
  \begin{scope}[yshift=1.5cm]
    \draw[-, thick] (-0.5,0) -- (10.5,0);
    \node at (0,-0.35) {\small $0$};
    \node at (10,-0.35) {\small $1$};
    \fill[gray!20]  (0,0.15)   rectangle (6.4,0.75);
    \fill[green!30] (6.4,0.15) rectangle (10,0.75);
    \draw (0,0.15)   rectangle (6.4,0.75);
    \draw (6.4,0.15) rectangle (10,0.75);
    \node at (2.5,0.45) {\small Do not target};
    \node at (8.2,0.45) {\small AI targeting};
    \foreach \x in {0, 6.4, 10}{\draw (\x,0) -- (\x,-0.15);}
    \node at (6.4,-0.38) {\small $\widetilde{q}_{\beta}$};
    \node[anchor=east] at (-0.5, 0.45) {\small AI-driven Allocation:};
  \end{scope}

  % ── BOTTOM DIAGRAM ───────────────────────────────────────────
  \begin{scope}[yshift=0cm]
    \draw[-, thick] (-0.5,0) -- (10.5,0);
    \node at (0,-0.35) {\small $0$};
    \node at (10,-0.35) {\small $1$};
    \fill[gray!20]   (0,0.15)  rectangle (5,0.75);
    \fill[orange!30] (5,0.15)  rectangle (7,0.75);
    \fill[green!30]  (7,0.15)  rectangle (10,0.75);
    \draw (0,0.15)  rectangle (5,0.75);
    \draw (5,0.15)  rectangle (7,0.75);
    \draw (7,0.15)  rectangle (10,0.75);
    \node at (2.5,0.45) {\small Do not target};
    \node at (6,0.45)   {\small Screen};
    \node at (8.5,0.45) {\small AI targeting};
    \foreach \x in {0, 5, 7, 10}{\draw (\x,0) -- (\x,-0.15);}
    \node at (5,-0.38)  {\small $q_{\alpha}$};
    \node at (7,-0.38)  {\small $q_{\beta}$};
    \node at (5,-0.95) {\small Low risk \quad $\leftarrow$\quad\quad Risk score $\mu(X) = \mathbb{P}(Y=1 \mid X)$ \quad\quad $\rightarrow$ \quad High risk};
    \node[anchor=east] at (-0.5, 0.45) {\small Screening + AI Allocation:};
  \end{scope}

\end{tikzpicture}
\caption{The population is ordered left to right by increasing vulnerability risk score $\mu(X) = \mathbb{E}[Y | X]$. \textit{Top:} AI-driven allocation directly targets units with $\mu(X) > \widetilde{q}_{\beta}$. \textit{Bottom:} Two-stage allocation with optimal screening band $[q_{\alpha},q_{\beta}]$, where units are allocated only if their true vulnerability status $Y=1$ is confirmed, and remaining budget is allocated to the highest-risk unscreened units, $\mu(X) > q_{\beta}$. 
The optimal screening set $[q_{\alpha},q_{\beta}]$ sits at the margin of algorithmic allocation around $\widetilde{q}_{\beta}$.}
\label{fig:screening_policy}
\end{figure}

Our contributions are as follows:
\begin{itemize}
    \item We characterize the optimal screening strategy in a two-stage allocation framework with Bayes-optimal risk. We show that, regardless of the underlying risk distribution, the optimal screening set is an interval targeting units \textit{at the margin of algorithmic allocation}, with endpoints that depend jointly on the risk distribution and the resource constraints, and develop a fixed-point algorithm that recovers the optimal band for arbitrary distributions.

    %\item We show that, regardless of the underlying risk distribution, the optimal screening set has an interval structure $[q_{\alpha}, q_{\beta}]$ and the pure AI allocation threshold $\widetilde{q}_{\beta}$ always falls within it.

    \item We establish a distribution-free characterization of the value of screening in allocation efficiency, formalizing a law of diminishing returns, and obtain sharp bounds on the worst-case loss of any fixed heuristic screening interval with respect to the optimal screening set. \looseness=-1

    \item We derive the relative value of screening and direct allocation, leading to a simple, distribution-free, and scale-invariant cutoff for splitting a program budget between the two. \looseness=-1
    
    \item We provide an empirical characterization of the relationship between aleatoric uncertainty and screening: the higher the irreducible uncertainty in the population, the greater the marginal gains from screening over purely algorithmic targeting.
    
    \item We demonstrate the practical effectiveness of our framework through real-world applications in humanitarian demining operations and income-based social programs, showing substantial gains over purely algorithmic targeting across budget regimes and risk estimation methods.
    
\end{itemize}

Our work characterizes the limits of AI-driven allocation under aleatoric uncertainty and provides decision-makers with concrete tools for optimally combining algorithmic targeting and screening. We release our code at \url{https://github.com/cmpatino/optimal-screening} so practitioners can readily apply the framework to their own risk scores, recovering the optimal screening and allocation policy under any budget constraints $\alpha$ and $\beta$.

\section{Related work}

Predictive machine learning is increasingly used to target scarce resources, including proxy means tests for poverty alleviation \cite{grosh1995proxy,aiken2022machine}, risk-based prioritization in humanitarian demining \cite{dulce2024reland,cirillo2024desk}, AI-assisted triage and health-care resource allocation \cite{wu2023aihealth,grandclement2024ventilator}, and decision support in education \cite{liu2025mladmission,perdomo2025difficult}. These applications typically frame prediction as a lower-cost substitute for direct verification \cite{grosh1995proxy,liu2025mladmission, dulce2024reland}. We instead study screening and algorithmic targeting as complementary instruments, with their complementarity driven by aleatoric uncertainty. Prior work distinguishes aleatoric from epistemic uncertainty, with the former reflecting stochasticity inherent in the data-generating process \cite{kendall2017uncertainties,hullermeier2021aleatoric}, while recent work argues that this dichotomy can be insufficiently expressive for some decision-relevant uncertainty quantities \cite{bickfordsmith2025rethinking,gruber2026sources}. We contribute to this discussion by linking aleatoric uncertainty to resource allocation directly, formalizing how its population-level magnitude determines the marginal value of screening and thereby when screening is most needed.

Other levers for improving allocation have also been studied. The idea of screening before allocating has been explored as a form of information acquisition: Cai et al.~\cite{cai2020fair} study a budget-constrained decision-maker who screens selected individuals to acquire additional features that refine their risk estimates before allocating a resource. In contrast, we take the risk to be Bayes-optimal given the available covariates, so that screening observes true outcomes and fully resolves the residual aleatoric uncertainty for screened units, which allows us to characterize the optimal policy analytically. Relatedly, Perdomo~\cite{perdomo2023relative} and Fischer-Abaigar et al.~\cite{fischer2025value} compare the relative value of improving the predictive model versus expanding access to the resource. We instead hold the model fixed at the Bayes-optimal limit and study targeted screening as a complementary lever that policymakers have long used in practice \cite{grosh1995proxy,who2020screening}. Our analysis of the budget split similarly endogenizes resources, but between screening and direct allocation rather than between prediction and access.

% Beyond the theoretical characterization, we contribute a fixed-point algorithm to operationalize the optimal screening policy for arbitrary risk distributions, and provide empirical evidence on the marginal value of screening as a function of aleatoric uncertainty. These results can inform the choice of screening budget in practical policy applications to improve resource allocation efficiency.

\section{Setup: Resource Allocation under Aleatoric Uncertainty}

Consider a centralized decision-maker that must allocate a scarce resource across a population to maximize allocation value. Each unit is characterized by covariates $X \in \mathcal{X}$ and a binary \textit{vulnerability status} $Y \in \{0,1\}$, where $Y=1$ indicates that the unit would benefit from the resource. Crucially, $Y$ is a fixed but latent characteristic of each unit, encoding a pre-existing condition rather than a response to the allocation decision. For instance, $Y=1$ may indicate the presence of a disease,  landmine contamination, or household poverty. 

The decision-maker observes $X$ for all units but may observe $Y$ only for those selected for physical verification, which we refer to as \textit{screening}. Units that are not screened must be allocated based on $X$ alone. A resource allocation rule is a measurable function $T: \mathcal{X} \to \{0,1\}$ that maps covariates to allocation decisions, where $T(X)=1$ indicates that a unit with features $X$ receives the resource. The \textit{vulnerability risk score} summarizes all available information about a unit's latent status:
\begin{equation}
  \mu(X) = \mathbb{E}[Y \mid X] = \mathbb{P}(Y=1 \mid X).
\end{equation}
We assume $\mu$ is Bayes-optimal given the available covariates unless explicitly stated otherwise.

Let $\beta \in (0,1)$ denote the allocation budget as a fixed fraction of the population. The objective is to maximize allocation efficiency, which we define here as the probability of a correct allocation to vulnerable units:
\begin{equation}
  \PP(T(X) = 1, Y = 1) = \mathbb{E}[T(X)Y] = \mathbb{E}[T(X)\mu(X)],
\end{equation}
subject to the budget constraint $\mathbb{E}[T(X)] \leq \beta$. Appendix \ref{app:notation} summarizes the notation used throughout.

\begin{remark}
    We assume, for now, that the allocation budget is fixed and the risk score is Bayes-optimal, reflecting short-term constraints where resources cannot be rapidly expanded and the model cannot be improved, so that all gains in allocation efficiency come from screening alone. This allows us to characterize, in full generality, the policy that maximizes allocation value for any given screening budget. Section~\ref{sec:budget_split} relaxes the first assumption, splitting a total budget between screening and allocation. \looseness=-1
\end{remark}

\paragraph{Aleatoric Uncertainty.} In our setting, a unit's vulnerability status follows $Y \mid X \sim \text{Bernoulli}(\mu(X))$. That is, outcomes are inherently stochastic for any unit with $\mu(X) \in (0,1)$. This irreducible randomness is the \textit{aleatoric uncertainty} of the problem, reflecting stochasticity in the outcome itself rather than estimation error \cite{kendall2017uncertainties,hullermeier2021aleatoric,gruber2026sources}, and cannot be reduced by collecting more data or improving the risk model.

Crucially, aleatoric uncertainty imposes a fundamental limit on the allocation efficiency of purely algorithmic targeting. Any \textit{ex ante} rule $T(X)$ based solely on $\mu(X)$ targets units that have $Y=1$ only in expectation, so some misallocation is inevitable. The realized allocation value $\mathbb{E}[T(X)Y]$ will therefore always be smaller than the \textit{ex post} optimal allocation where $Y$ is known before the decision is made, even when $\mu$ is Bayes-optimal. 

The magnitude of this gap depends on the distribution of $\mu(X)$ and in particular on its aleatoric uncertainty. When $\mu(X) \sim \delta_{0.5}$, risk scores are uninformative with high aleatoric uncertainty, and optimal allocation precision is at most 50\%. When $\mu(X)$ follows a perfectly separating bimodal distribution, risk is well-separated with low aleatoric uncertainty, and algorithmic targeting can achieve full allocation efficiency. This motivates screening as a mechanism to close this gap.

\paragraph{Screening.} We define \textit{screening} as a procedure to directly observe the realized vulnerability status of any unit of interest, resolving aleatoric uncertainty at a cost. Formally, a screening rule is a function $S: \mathcal{X} \to \{0,1\}$, where $S(X) = 1$ indicates that a unit with covariates $X$ is selected for physical verification. A decision-maker with a screening budget $\alpha \in (0, \beta)$ can observe the true outcome $Y$ for a fraction $\alpha$ of the population before the allocation decision is made, converting a uncertain allocation into a certain one for those units. In the next section, we formalize this two-stage framework for optimally combining screening and algorithmic targeting to maximize allocation efficiency.

\section{Resource Allocation with Screening}

While universal screening would fully resolve aleatoric uncertainty, it is prohibitively costly in practice. Instead, we consider a two-stage framework in which a decision-maker operates under  a screening budget $\alpha \in (0,\beta)$ and allocation budget $\beta \in (0,1)$, jointly optimizing over the screening $S$ and allocation $T$ rules to maximize allocation efficiency. The framework proceeds as follows:

\begin{enumerate}
    \item \textbf{Screening stage}: Select a subset of units $S(X) = 1$ to observe their true vulnerability status $Y$, subject to a screening budget $\mathbb{E}[S(X)] \leq \alpha$.
    \item \textbf{Allocation stage}: Assign the resource based on both covariates $X$ and observed outcomes for screened units, subject to the allocation budget $\beta$. Screened units with $Y=1$ are allocated with certainty, which is always ex-post optimal, and the remaining budget is deployed to unscreened units based on $\mu(X)$ alone.
\end{enumerate}

This defines the two-stage allocation rule
\begin{equation}
\label{eq:two_stage_rule}
    \pi(X, Y) = S(X)Y + (1 - S(X))T(X),
\end{equation}
where $Y$ is only observed for screened units $S(X) = 1$, and $T(X)$ is the allocation rule for unscreened units based on $\mu(X)$ alone. A decision-maker chooses $S$ and $T$ to maximize allocation efficiency as the probability of allocation to vulnerable units, subject to screening and allocation budgets:
% \begin{align}
% \max_{S: \mathcal{X} \to \{0,1\}, ~ T: \mathcal{X} \to \{0,1\}} &\quad \mathbb{E}[\pi(X,Y)Y] \label{eq:full_objective}\\
% \text{s.t.} &\quad \mathbb{E}[S(X)] \leq \alpha \nonumber \\
% &\quad \mathbb{E}[S(X)Y] + \mathbb{E}[(1-S(X))T(X)] \leq \beta \nonumber.
% \end{align}
\begin{align}
\max_{S, T: ~ \mathcal{X} \to \{0,1\}} &\quad \mathbb{E}[\pi(X,Y)Y] \label{eq:full_objective}\\
\text{subject to} &\quad \mathbb{E}[S(X)] \leq \alpha, \quad\quad \mathbb{E}[S(X)Y] + \mathbb{E}[(1-S(X))T(X)] \leq \beta. \nonumber
\end{align}

We study this framework by first characterizing the optimal policy in the absence of screening, then deriving the optimal screening rule for a given $\alpha$. We finalize with a distribution-free characterization of the value of screening and  a formal law of diminishing returns to screening in allocation efficiency.

\subsection{Optimal Algorithmic Allocation}

We first consider the case of purely algorithmic allocation, where $\alpha = 0$ and screening is unavailable. Thus, all allocation decisions are made ex ante based solely on $\mu(X)$, and the problem reduces to
\begin{equation}
\label{eq:objective_no_screen}
\max_{T:~ \mathcal{X} \to \{0,1\}} \quad \mathbb{E}[T(X)\mu(X)] \quad \text{subject to} \quad \mathbb{E}[T(X)] \leq \beta.
\end{equation}

The solution has a simple and intuitive form: allocate to units in decreasing order of their risk score $\mu(X)$ until the budget is exhausted. Note that analogous threshold rules have been derived in related constrained allocation settings \citep{kitagawa2018should, luedtke2016optimal, fischer2025value, cai2020fair}. All proofs are deferred to Appendix~\ref{app:proofs}.

\begin{proposition}[Optimal Allocation Rule without Screening]
\label{prop:optimal_no_screening}
The optimal allocation rule for \eqref{eq:objective_no_screen} is
$T^*(X) = \mathbf{1}\{\mu(X) > \widetilde{q}_\beta\},$
where $\widetilde{q}_\beta$ is the $(1-\beta)$-quantile of the distribution of $\mu(X)$ for $\beta \in (0,1)$.
\end{proposition}

%We denote the welfare under optimal algorithmic allocation as $V^*(0) = \mathbb{E}[T^*(X)\mu(X)]$, which serves as the baseline against which we measure the gains from screening.

\subsection{Optimal Screening Rule}

We now characterize the optimal solution to the two-stage allocation problem. A critical distinction from the purely algorithmic allocation case is that screening resolves aleatoric uncertainty before allocation so that the realized vulnerability status $Y$ is observed for screened units \textit{before final allocation decisions are made}. From \eqref{eq:two_stage_rule}, this yields the allocation rule $\pi(X,Y) = Y$ for screened units ($S(X) = 1$), where $Y$ is a deterministic observed value rather than a random variable. Intuitively, screening improves allocation efficiency in two ways: screened units with $Y=1$ are allocated with certainty, improving targeting precision, while screened units with $Y=0$ are excluded from allocation, avoiding wasteful use of the fixed budget $\beta$.

We revisit the two-stage optimization problem \eqref{eq:full_objective} by substituting the two-stage allocation rule \eqref{eq:two_stage_rule} into the objective, decomposing $\mathbb{E}[\pi(X,Y)Y]$ by whether units are screened or not:
\begin{align}
\mathbb{E}[\pi(X,Y)Y] &= \mathbb{E}[S(X)Y] + \mathbb{E}[(1-S(X))T(X)Y]  \\
&= \mathbb{E}[S(X)Y] + \mathbb{E}[(1-S(X))T(X)\mu(X)] \nonumber.
\end{align}
The first term captures allocation efficiency among screened and targeted units, and the second from algorithmic targeting of unscreened units. Equality follows from iterated expectations.

The allocation budget constraint also decomposes accordingly: screened units are allocated if and only if $Y=1$, using $\mathbb{E}[S(X)Y]$ of the budget, while unscreened units are allocated via $T(X)$, using $\mathbb{E}[(1-S(X))T(X)]$. The full two-stage optimization problem is therefore
% \begin{align}
% \label{eq:two_stage_opt}
% \max_{S: \mathcal{X} \to \{0,1\}, ~ T: \mathcal{X} \to \{0,1\}} \quad 
%     & \mathbb{E}[S(X) Y] + \mathbb{E}[(1-S(X)) T(X) \mu(X)] \\
% \text{s.t.}\quad
%     & \mathbb{E}[S(X)] \leq \alpha, \nonumber \\
%     & \mathbb{E}[S(X) Y] + \mathbb{E}[(1-S(X)) T(X)] \leq \beta, \nonumber
% \end{align}
\begin{align}
\label{eq:two_stage_opt}
\max_{S, T: ~\mathcal{X} \to \{0,1\}} \quad 
    & \mathbb{E}[S(X) Y] + \mathbb{E}[(1-S(X)) T(X) \mu(X)] \\
\text{subject to} \quad
    & \mathbb{E}[S(X)] \leq \alpha, \quad\quad \mathbb{E}[S(X) Y] + \mathbb{E}[(1-S(X)) T(X)] \leq \beta, \nonumber
\end{align}
where $\alpha$ and $\beta$ are the screening and allocation budgets respectively.

The central question is: \textit{which units should be screened to optimize allocation efficiency?} Unlike the no-screening case, where the optimal rule is a simple quantile threshold (Proposition~\ref{prop:optimal_no_screening}), the interaction between screening and allocation budgets creates a more complex optimization problem. Each screened unit uses allocation budget with probability $\mu(X)$, since it is allocated only when $Y=1$, so screening the highest-risk units is inefficient and these are better served by direct allocation. Screening the lowest-risk units is equally inefficient since finding $Y=1$ among them is rare, yielding few confirmed allocations. The optimal screening set therefore targets units in between: those with sufficient uncertainty that observing $Y$ meaningfully improves allocation decisions, \textit{and} sufficient risk to be competitive candidates given the budget $\beta$.

The characterization of the optimal screening rule builds on the intuition of Proposition~\ref{prop:optimal_no_screening}: optimal policies are naturally expressed as sets defined by quantile thresholds in the risk score space. Since quantile notation will be central to all subsequent results, we establish it here alongside the relevant distributional objects. Let $F$ denote the CDF of $\mu(X)$ and $F^{-1}(q)$ its $q$-th quantile, so that an interval $[q_a, q_b]$ refers to all units whose risk score falls between the $q_a$-th and $q_b$-th quantiles of $F$, with higher quantiles corresponding to higher risk. We now state our main theoretical result.

%One advantage of working in the risk space is that risks themselves are one-dimensional. Since all one-dimensional convex sets are intervals, convexity of the screening and treatment sets is equivalent to each set is fully characterized by just two quantile thresholds. We now show this convexity holds at optimality.

\begin{theorem}[Optimal Screening in Two-Stage Allocation]
\label{thm:optimal_screening}
Let $\alpha \in (0,\beta)$ and $\beta \in (0,1)$ with $\alpha + \beta < 1$, and assume the risk distribution has full support on $[0,1]$. The optimal screening rule is $S^*(X) = \mathbbm{1}\{\mu(X) \in [q_\alpha, q_\beta]\}$
and the optimal direct allocation rule is $
    T^*(X) = \mathbbm{1}\{\mu(X) > q_\beta\}$
where the thresholds $0 < q_\alpha \leq q_\beta < 1$ satisfy the screening and allocation budget constraints:
\begin{equation}
F(q_\beta) - F(q_\alpha) = \alpha \qquad  \text{and} \qquad \int_{q_\alpha}^{q_\beta} \mu \, dF(\mu) + 1 - F(q_\beta) = \beta.
\label{eq:budget_constraints}
\end{equation}
\end{theorem}

\begin{proof}[Proof sketch]
The proof proceeds in three steps. Full details are presented in Appendix \ref{app:proofs}.

\textit{Step 1: No directly allocated unit lies below a screened one.} Suppose unit $z$ with $\mu(z) < \mu(z')$ is directly allocated while unit $z'$ is screened. Swapping them preserves the objective value and screening budget but reduces the allocation budget used by screening, since screening $z$ uses allocation budget with probability $\mu(z) < \mu(z')$. This extra saved budget allows extending the direct allocation band to additional units, strictly improving the objective and contradicting optimality.

\textit{Step 2: Given $S$, the optimal allocation rule is a threshold in $\mu(X)$.} The ``second-stage'' unscreened allocation problem reduces to the no-screening problem of Proposition \ref{prop:optimal_no_screening} with residual budget $\beta - \mathbb{E}[S(X)Y]$. Since Step 1 implies the screening set lies entirely below the direct allocation set, the optimal unscreened allocation rule is a threshold $T^*(X) = \mathbf{1}\{\mu(X) > q_\beta\}$ in the original risk score space, with the highest-risk units always directly allocated.

\textit{Step 3: The optimal screening interval is $[q_\alpha, q_\beta]$.} Among all screening intervals of mass $\alpha$, shifting the screening interval upward has two effects: a gain in allocation efficiency from screening higher-risk units and a loss from the induced tightening of the direct allocation budget. The gain exceeds the loss because the reduction in the direct allocation term is proportional to $q_\beta < 1$, while the gain is proportional to the risk of the newly screened units. The net effect is always positive, so the objective is strictly increasing in the upper boundary. The optimal screening interval is therefore pushed up until $q_\beta$ solves the budget constraints \eqref{eq:budget_constraints}, with lower boundary $q_\alpha = F^{-1}(F(q_\beta) - \alpha)$ determined by the screening budget constraint.
\end{proof}

Theorem \ref{thm:optimal_screening} formally shows that the optimal screening set targets the highest-uncertainty units among relevant allocation candidates. The screening band $[q_\alpha, q_\beta]$ sits immediately below the direct allocation band $[q_\beta, 1]$, where units above $q_\beta$ have sufficiently high risk to warrant direct allocation, while units below $q_\alpha$ are too unlikely to have $Y=1$ to justify the screening cost. Corollary \ref{cor:at_the_margin} further shows that the screening set targets units \textit{at the margin of algorithmic allocation}: the purely algorithmic targeting threshold $\widetilde{q}_\beta$ from Proposition~\ref{prop:optimal_no_screening} is contained in the optimal screening interval $[q_\alpha, q_\beta]$.

\begin{corollary}
\label{cor:at_the_margin}
The purely algorithmic allocation threshold $\widetilde{q}_\beta$ falls within the optimal screening band, $q_\alpha < \widetilde{q}_\beta < q_\beta$, for any $\alpha \in (0,\beta)$. Moreover, $q_\beta$ is increasing in $\alpha$ and $q_\alpha$ is decreasing in $\alpha$, expanding the screening band toward higher- and lower-risk units as the screening budget grows.
\end{corollary}

Notably, the structure of the optimal policy does not depend on the underlying risk distribution $F$ directly. The screening and allocation bands are always contiguous intervals in the risk score space, with boundaries determined solely by the budget constraints \eqref{eq:budget_constraints}. The distribution $F$ enters only implicitly through the quantiles $q_\alpha$ and $q_\beta$.\footnote{For the special case of uniform risk, the optimal thresholds can be solved analytically; see Corollary~\ref{cor:uniform} in Appendix~\ref{app:proofs}.} In the next section, we leverage this result to establish a distribution-free characterization of the value of screening.

\begin{remark}
Since the optimal thresholds $q_\alpha$ and $q_\beta$ are defined in risk score space, the result holds regardless of the nature of $X$, whether structured tabular data, images, text, or any other modality.
\end{remark}

\subsection{Value of Screening}
\label{sec:marginal_value}

To study how allocation efficiency improves as the screening budget $\alpha$ grows, first note that, under the optimal two-stage allocation rule, the optimal allocation value takes the form
\begin{equation}\label{eq:value}
    V^*(\alpha, \beta) = \int_{q_\alpha}^{1} \mu \, dF(\mu),
\end{equation}
where $q_\alpha = F^{-1}(F(q_\beta) - \alpha)$ is determined by the screening budget constraint, and $q_\beta$ solves the allocation budget constraint \eqref{eq:budget_constraints}.

\begin{theorem}[Marginal Value of Screening]
\label{thm:value_screening}
The marginal value of the screening budget is $\frac{dV^*(\alpha, \beta)}{d\alpha} = \frac{q_\alpha(1 - q_\beta)}{1 - q_\beta + q_\alpha} > 0,$ and $V^*(\alpha, \beta)$ is concave in $\alpha$.
\end{theorem}

Theorem~\ref{thm:value_screening} formally establishes that screening always improves optimal allocation efficiency, with $dV^*/d\alpha > 0$ for any feasible $\alpha$. Moreover, the marginal value of screening is decreasing in $\alpha$, formalizing a law of diminishing returns to physical verification. This result is distribution-free: the marginal value depends only on the optimal thresholds $q_\alpha$ and $q_\beta$, with the risk distribution $F$ entering only implicitly through the budget constraints that determine the screening set.

The value of screening depends not only on \textit{how much} is screened but also on \textit{where}. Appendix~\ref{app:heuristics} derives a sharp worst-case bound showing that fixed heuristic screening intervals, such as one centered at the allocation margin, can lose a substantial share of this value relative to the optimal band.

\subsection{Relative Value of Screening and Direct Allocation}
\label{sec:budget_split}

Having characterized the value of the optimal screening policy, we now study when screening is beneficial relative to direct allocation and derive the optimal split of a shared budget between screening and allocation. That is, we endogenize the budgets $(\alpha, \beta)$, which so far have been treated as exogenous parameters of the model. We begin with a cost-free characterization of the trade-off between the two.\looseness=-1

\begin{theorem}[Marginal Rate of Substitution between Screening and Allocation]
\label{thm:mrs}
For any $0 < \alpha < \beta < 1$, the marginal value of the allocation budget is $\frac{dV^*(\alpha, \beta)}{d\beta} = \frac{q_\alpha}{1 - q_\beta + q_\alpha}$, and the marginal rate of substitution between screening and allocation is $\frac{dV^*/d\alpha}{dV^*/d\beta} = 1 - q_\beta.$
\end{theorem}

Theorem~\ref{thm:mrs} states that one unit of screening is worth $1 - q_\beta$ units of direct allocation. Intuitively, the marginal unit at the direct allocation threshold has risk $q_\beta$, so allocating to it directly wastes the resource with probability $1 - q_\beta$, whereas screening it first allocates only if it is positive. When $q_\beta$ is low, direct allocation at the margin reaches low-risk units and screening is relatively valuable; as $q_\beta \to 1$, the marginal directly allocated unit is high-risk and screening it reveals little. Notably, this rate depends neither on $q_\alpha$ nor on the risk distribution $F$, except through $q_\beta$.

Next, let $c_s$ and $c_a$ denote the per-unit costs of screening and allocation, respectively, and let $C$ denote the total budget per member of the population, consistent with $(\alpha, \beta)$ being fractions of the same population. The decision-maker then chooses the budgets that maximize the optimal allocation value:
\begin{equation}\label{eq:budget_problem}
    \max_{\alpha, \beta} \; V^*(\alpha, \beta) \quad \text{subject to} \quad c_s \alpha + c_a \beta \leq C.
\end{equation}
Since the marginal values of screening and allocation are positive (Theorems~\ref{thm:value_screening} and~\ref{thm:mrs}), the solution to~\eqref{eq:budget_problem} exhausts the full budget. Moreover, at an interior solution, where both instruments are used ($0 < \alpha^* < \beta^* < 1$), the marginal value per dollar must equalize across the two instruments, $\frac{dV^*/d\alpha}{c_s} = \frac{dV^*/d\beta}{c_a}$. Equivalently, the marginal rate of substitution must equal the cost ratio. \looseness=-1

\begin{corollary}[Optimal Budget Split between Screening and Allocation]
\label{cor:cost_ratio}
Consider problem~\eqref{eq:budget_problem} with $0 < c_s < c_a$. If $C \geq c_a \left(1 - F(1 - c_s / c_a)\right)$, the optimal direct allocation threshold is $q_\beta^* = 1 - \frac{c_s}{c_a}$, which determines the optimal screening and allocation budgets $(\alpha^*, \beta^*)$ via the binding total budget constraint. Otherwise, the optimal policy allocates the entire budget directly, $\alpha^* = 0$ and $\beta^* = C / c_a$.
\end{corollary}

Corollary~\ref{cor:cost_ratio} gives a simple, distribution-free, and scale-invariant condition: the optimal boundary between the screening band and direct allocation depends only on the cost ratio $c_s / c_a$. Intuitively, the cost ratio determines \emph{where} to draw the line in risk space, while $C$ and the risk distribution determine \emph{how much} this translates into in absolute terms. The budget condition simply requires that there is enough budget to allocate directly to every unit above the cutoff, whose mass is $1 - F(q_\beta^*)$. Otherwise, the budget is exhausted before the threshold $q_\beta^*$ is reached. 

The cost ratio also clarifies when screening is worthwhile. If $c_s \geq c_a$, $q_\beta^*$ is not a valid quantile and the optimum is not to screen at all, consistent with a unit of screening never being worth more than a unit of direct allocation (Theorem~\ref{thm:mrs}). As $c_s / c_a \to 0$, screening extends to high-risk levels, with $\alpha \to \beta$. This is the regime considered in our experiments below, where screening is cheap relative to allocation. To decide whether additional screening is the best use of available funds, a decision-maker need only compute $c_s / c_a$ for their setting, and the framework determines how aggressively to screen.

%\textit{Remark: This reconciles with the previous result from the uniform distribution, the alpha values derived as critical points yields to either $q_S = 0$ or $q_B = 1$, which are trivial critical point where we screen all the population.}

%\textit{Remark 2: This expression \eqref{eq:dV_dalpha} is distribution-free: the marginal value of screening depends only on the optimal screening cutoff values $q_B^*$ and $q_S^*$. The distribution $F$ enters only implicitly through the these cutoffs via the screening and budget constraints.}
%\textit{TODO: We can also get the second derivative wrt $\alpha$ and I think is always negative so diminishing returns as expected!}
%\textit{Remark 72: Be clear about our setting predicting current outcomes and not future not-realized outcomes. Done in the setup}

\section{Computing the Optimal Screening Quantiles}
\label{sec:algorithm}

We develop a fixed-point algorithm to recover the optimal screening quantiles via Algorithm~\ref{algo} in Appendix \ref{app:algorithm}. While Theorem \ref{thm:optimal_screening} characterizes the structure of the optimal screening rule, it does not directly yield a constructive method for computing the band thresholds, as $q_\beta$ requires solving a nonlinear integral equation with no closed-form solution in general. The algorithm partitions the population into three contiguous bands ordered by risk. The top band, containing the highest-risk fraction $\beta - \alpha$ of the population, is fixed and directly allocated without screening. Below it, the algorithm iterates over two bands: a screening band of mass $\alpha$ and a residual allocation band of directly allocated units using the budget saved by excluding confirmed $Y=0$ units within the screening band. The key insight is that the saved budget depends on the average risk $\rho = \frac{1}{\alpha}\int_{q_\alpha}^{q_\beta}\mu\,dF(\mu)$ within the screening band: the lower the average risk of screened units, the more $Y=0$ units are excluded and the larger the saved budget $\alpha(1-\rho)$. Given a candidate $\rho^{(k)}$, the algorithm fixes the residual allocation band mass at $\alpha(1-\rho^{(k)})$, places the screening band immediately below, and recomputes the average risk over the new screening band to obtain $\rho^{(k+1)} = g(\rho^{(k)})$. Theorem \ref{thm:fixed_point} in Appendix \ref{app:algorithm} proves that $g$ is a contraction, guaranteeing convergence to a fixed point $\rho^*$; $q_\beta$ is then recovered as the boundary between the residual allocation and screening bands as $q_\beta = F^{-1}(1-\beta + \alpha\rho^*)$.

\section{Experiments}

We validate our theoretical findings and illustrate the practical relevance of optimal screening in resource allocation using both synthetic and real-world data. We first characterize the gains from optimal screening across risk distributions with varying levels of aleatoric uncertainty, then present two real-world allocation settings: humanitarian demining operations in Colombia and identifying low-income individuals from the U.S. Census demographic data. All the code used for our experimental analysis is available at  \url{https://github.com/cmpatino/optimal-screening}.

\subsection{Empirical Value of Screening under Aleatoric Uncertainty}

Our theoretical results are distribution-free: regardless of the underlying risk distribution, the structure of the optimal screening set and the concavity of the value function with respect to the screening budget $\alpha$ hold. To illustrate this property and characterize how the value of screening varies with aleatoric uncertainty, we sweep over the one-parameter family $\mu \sim \text{Beta}(t,t)$, $t\geq 0$. All distributions in this family are symmetric around $\EE[\mu]= 0.5$, so differences in targeting outcomes are driven purely by distributional shape. Moreover, $\text{Beta}(t,t)$ continuously interpolates between canonical risk distributions: bimodal ($0 < t < 1$), uniform ($t=1$) and unimodal bell-shaped ($1 < t < \infty$), representing increasing aleatoric uncertainty regimes. In the limit $t \to 0$ the distribution concentrates around binary risk $\mu \in \{0, 1\}$, and as $t \to \infty$ it converges to the point mass $\delta_{1/2}$, where all units share the same risk $\mu = 1/2$. These represent the two extremes: one where underlying risk perfectly determines vulnerable units (no aleatoric uncertainty), and one where risk is entirely uninformative about vulnerability status (maximum aleatoric uncertainty). 

For a given $t$, we generate synthetic populations with underlying risk $\mu_i \sim \text{Beta}(t,t)$ and realized outcomes $Y_i \mid \mu_i \sim \text{Bernoulli}(\mu_i)$, and compute the optimal screening policy via Algorithm \ref{algo}. 

Figure \ref{fig:sim_noscreening} fixes $\beta = 35\%$ and shows screening gains for $\alpha \in [0, \beta]$ across four reference distributions: bimodal ($t=0.1$), uniform ($t=1)$, unimodal ($t=10)$, and $\delta_{0.5}$. Critically, without screening, optimal allocation precision ranges from 50\% to 100\%, demonstrating that even with access to the true risk distribution, targeting efficiency is fundamentally limited by aleatoric uncertainty. In particular, the bimodal distribution achieves perfect precision \emph{without any screening}, while optimal screening with $\alpha = \beta = 35\%$ raises precision from 82\% to 98\% for the uniform, from 61\% to 83\% for the unimodal, and from 50\% to 75\% for $\delta_{0.5}$. In the latter case, screening yields large efficiency gains even though the risk scores carry no targeting information whatsoever.  Crucially, since the total number of targeted units is fixed at $\beta$, all gains come from better targeting through observed outcomes via screening. Appendix~\ref{app:simulations} and \ref{app:heuristics} benchmark against random screening and fixed heuristic, showing that gains depend not only on how much screening is performed but on where it is concentrated.

Figure \ref{fig:utility_gap_vs_dist} fixes budgets $\alpha = \beta = 35\%$ and reports the \textit{utility gap} $\Delta V^*(\alpha) = V^*(\alpha) - V^*(0)$, the percentage-point improvement in precision from optimal screening relative to no-screening (Proposition \ref{prop:optimal_no_screening}), as a function of $t$. The experimental results demonstrate a systematic relationship between aleatoric uncertainty and the value of screening: $\Delta V^*(\alpha)$ increases monotonically with $t$, tracing an $S$-shaped curve with three distinct phases.
(i) As $t\to 0$, the marginal value of screening is flat and negligible: risk scores are highly informative and algorithmic targeting acts as a near-perfect substitute for screening. 
(ii) For $t \in [0.1, 10]$, spanning bimodal, uniform, and unimodal bell-shaped distributions, gains in allocation efficiency steeply increase as screening acts as a complement to algorithmic allocation.
(iii) As $t \to \infty$, the curve flattens again as risk becomes highly non-informative and gains from screening saturate at their maximum. 
Together, Figures~\ref{fig:sim_noscreening} and \ref{fig:utility_gap_vs_dist} characterize the conditions under which screening and algorithmic targeting act as substitutes or complements, with the value of screening growing monotonically with population-level aleatoric uncertainty.
In particular, real-world risk distributions with intermediate aleatoric uncertainty fall precisely in the steep region of the curve, where screening yields the largest marginal gains.

% More broadly, these results provide empirical evidence that algorithmic targeting and screening operate at different stages of uncertainty: prediction provides \textit{ex-ante} risk stratification, while screening enables \textit{ex-post} information acquisition. 

% \textit{Remark: We can plot where in this plot a bunch of datasets to predict social programs will fall! I think it would have to be in a different plot as these are not beta but compute the utility gap for ~20 datasets. Ideas: fit a beta distribution to each dataset and see where it falls here; or histogram of utility gap across dataset, or x axis being a measure of aleatoric uncertainty.}
%Idea: run some experiments where screening doesn't fully reveal the outcome but just give \textit{local} information about the screened units. Like an additional feature that you can use for prediction in those units (and collapses to our setting when that signal is $\pm \infty$), and check how this utility gap vs. beta plot would change. Is screening still that informative in this case?

% This sheds light on the limits of purely algorithmic allocation: even with access to the true risk distribution, targeting efficiency varies substantially with aleatoric uncertainty; degrading to a random allocation when $\mu \sim \delta_{0.5}$.

\begin{figure}[t]
    \begin{minipage}[t]{0.49\linewidth}
    \centering
    \includegraphics[width=\linewidth]{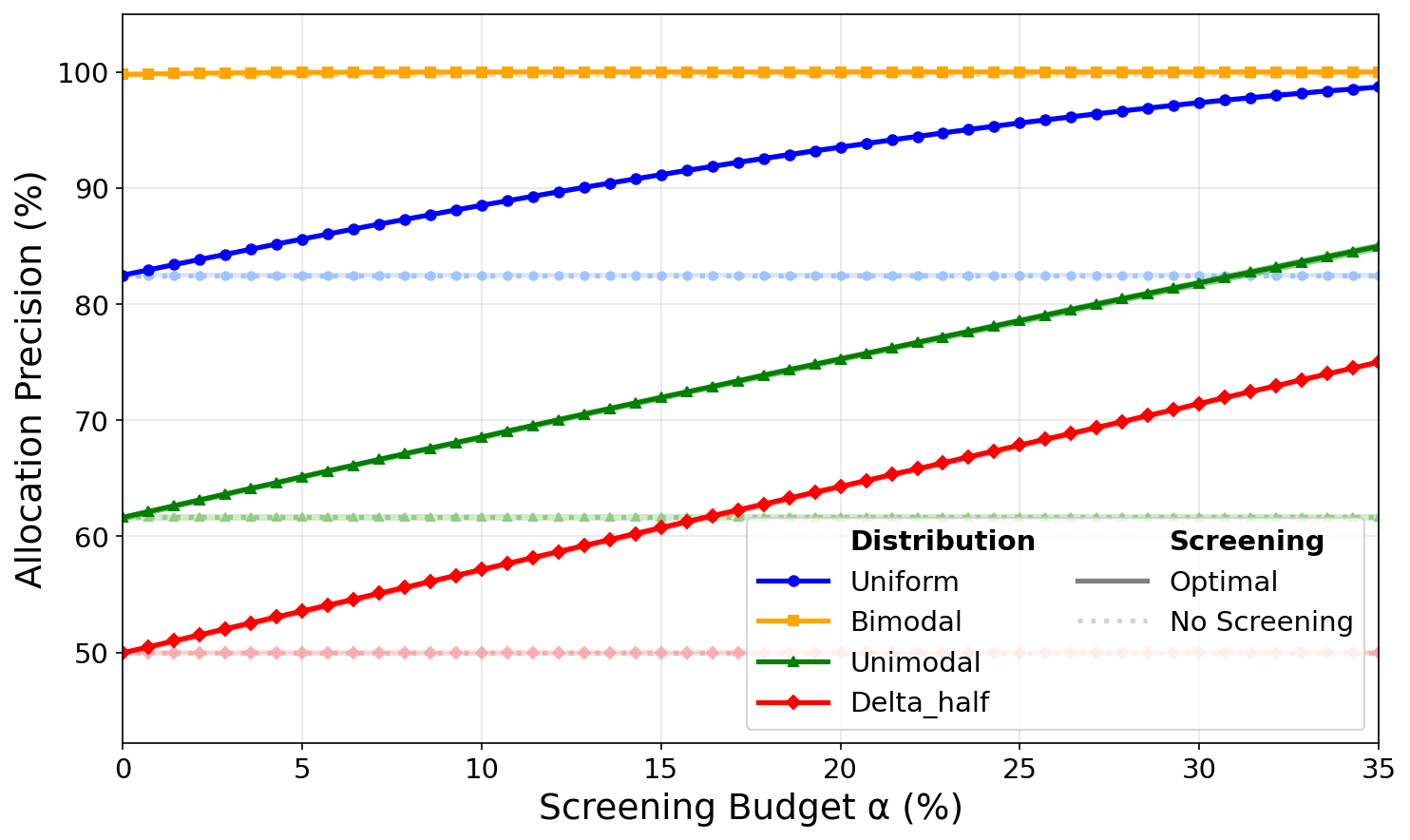}
    \caption{Allocation precision by screening budget $\alpha$ under optimal screening vs. no-screening across risk distributions ($\beta = 35\%$). Results are averaged over 10 simulations; standard deviations are negligible for every $\alpha$ (see Table~\ref{tab:heuristics}).}
    \label{fig:sim_noscreening}
    \end{minipage}
\hfill
    \begin{minipage}[t]{0.49\linewidth}
    \centering
    \includegraphics[width=\linewidth]{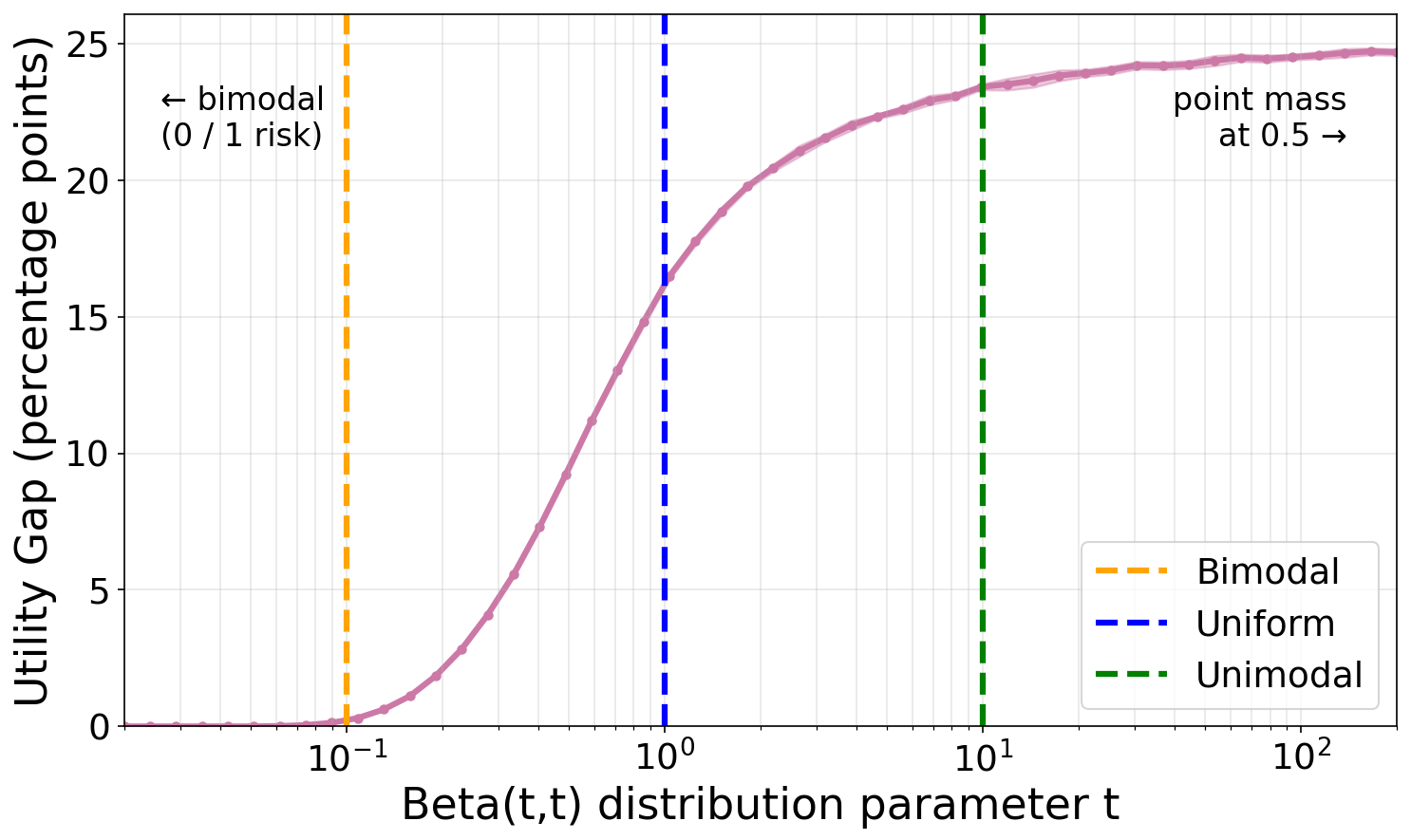}
    \caption{Utility gap $\Delta V^*(\alpha) = V^*(\alpha) - V^*(0)$ between optimal screening and no-screening as a function of the Beta$(t,t)$ concentration parameter $t$ ($\alpha=\beta=35\%$,  averaged over 10 simulations $\pm$ one standard deviation). }
    \label{fig:utility_gap_vs_dist}
    \end{minipage}
\end{figure}

\subsection{Applications to Real-World Resource Allocation} 
We apply our optimal screening framework to two real-world settings: humanitarian demining operations in Colombia and identifying low-income individuals for social protection programs in the US. In both settings, we train off-the-shelf machine learning models (logistic regression, random forest, and gradient boosting machine) on available predictive features and obtain out-of-sample risk estimates used to compute the optimal targeting policy for varying screening budgets $\alpha$. We compare against a random screening baseline that randomly observes outcomes and allocates greedily (Proposition \ref{prop:optimal_no_screening}) using predictions from logistic regression. Reproducibility details are provided in Appendix~\ref{app:reproducibility}. \looseness=-1

% collecting information from primary sources to determine contamination status--- from geographic, demographic, and conflict-related variables
\paragraph{Humanitarian Demining.} Humanitarian demining is the process of detecting and clearing landmines and other explosive remnants of war from conflict-affected areas so communities can safely reclaim land for economic and social development \cite{un_landmines}. A decision-maker must allocate limited specialized demining teams across a territory, but first needs to identify which areas are actually contaminated with explosive ordnance (EO). Traditional mine action operations do this by screening the entire territory, a costly and time-consuming procedure \cite{Harutyunyan2023OperationalEfficiency}. This has motivated recent proposals to estimate EO contamination probabilities using machine learning, and use these estimates to prioritize field screening. Such systems have shown promising results in guiding team allocation in Colombia \cite{dulce2024reland}, Afghanistan \cite{collins28advancements}, and Cambodia \cite{cirillo2024desk}. We apply our framework to study algorithmic allocation of demining teams in rural Colombia, using publicly available data from \cite{dulce2024reland}. The dataset includes labeled data from finalized demining operations, with 63 predictive features encompassing geographic, socio-demographic, and conflict-related characteristics. We obtain out-of-sample risk estimates for Granada municipality, where approximately 10\% of spatial units are truly contaminated. Accordingly, we set the allocation budget to $\beta = 10\%$, matching the underlying contamination prevalence.\footnote{This implies that an allocation procedure with access to true contamination status would achieve perfect recall. Moreover, since $\beta = \EE[Y]$, precision and recall coincide in this setting.} \looseness=-1

\paragraph{Income-Based Social Programs.} Income verification is commonly used as a gatekeeping mechanism for determining eligibility in social programs. However, collecting and verifying income information can be costly, creating a natural setting in which screening decisions must trade off verification costs against allocation accuracy. We deploy our framework to predict whether an individual's total personal income exceeds \$50{,}000 per year, using the 2018 American Community Survey (ACS) Public Use Microdata Sample, accessed via the \texttt{folktables} benchmark \citep{ding2021retiring}. This threshold follows the standard \texttt{ACSIncome} prediction task in this benchmark, rather than the eligibility criteria of any particular social program. The dataset contains approximately 1.6 million individual records drawn from the U.S. Census Bureau's annual survey of socioeconomic conditions, with ten demographic and labor-market features per record. We split the data evenly into training and test sets, stratified on the outcome, and fix the allocation budget at $\beta = 30\%$.

% : age, class of worker, educational attainment, marital status, occupation code, place of birth, relationship to household head, usual hours worked per week, sex, and race

\begin{figure}[t]
    \begin{minipage}[t]{0.49\linewidth}
    \centering
    \includegraphics[width=\linewidth]{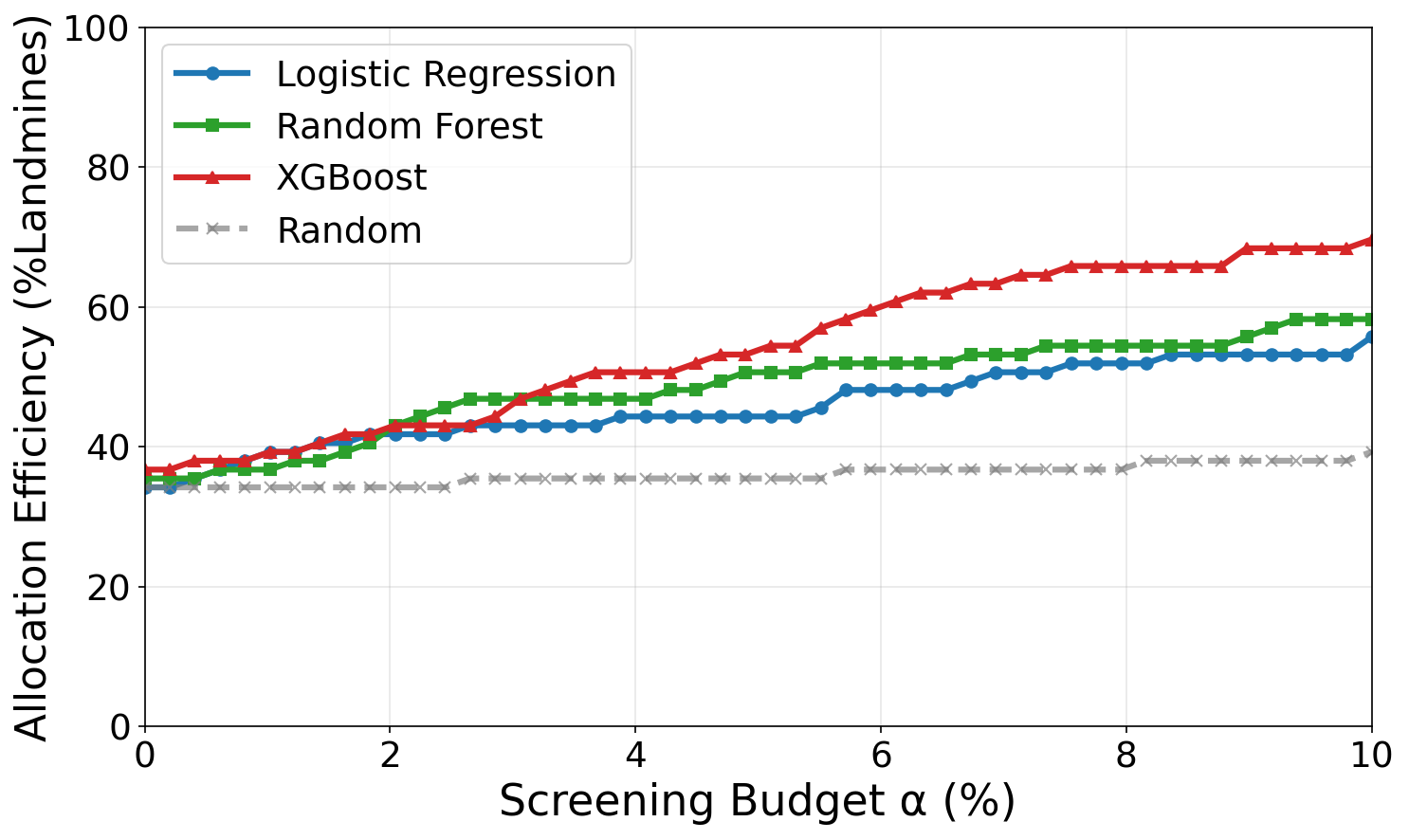}
    \end{minipage}
\hfill
    \begin{minipage}[t]{0.49\linewidth}
    \centering
    \includegraphics[width=\textwidth]{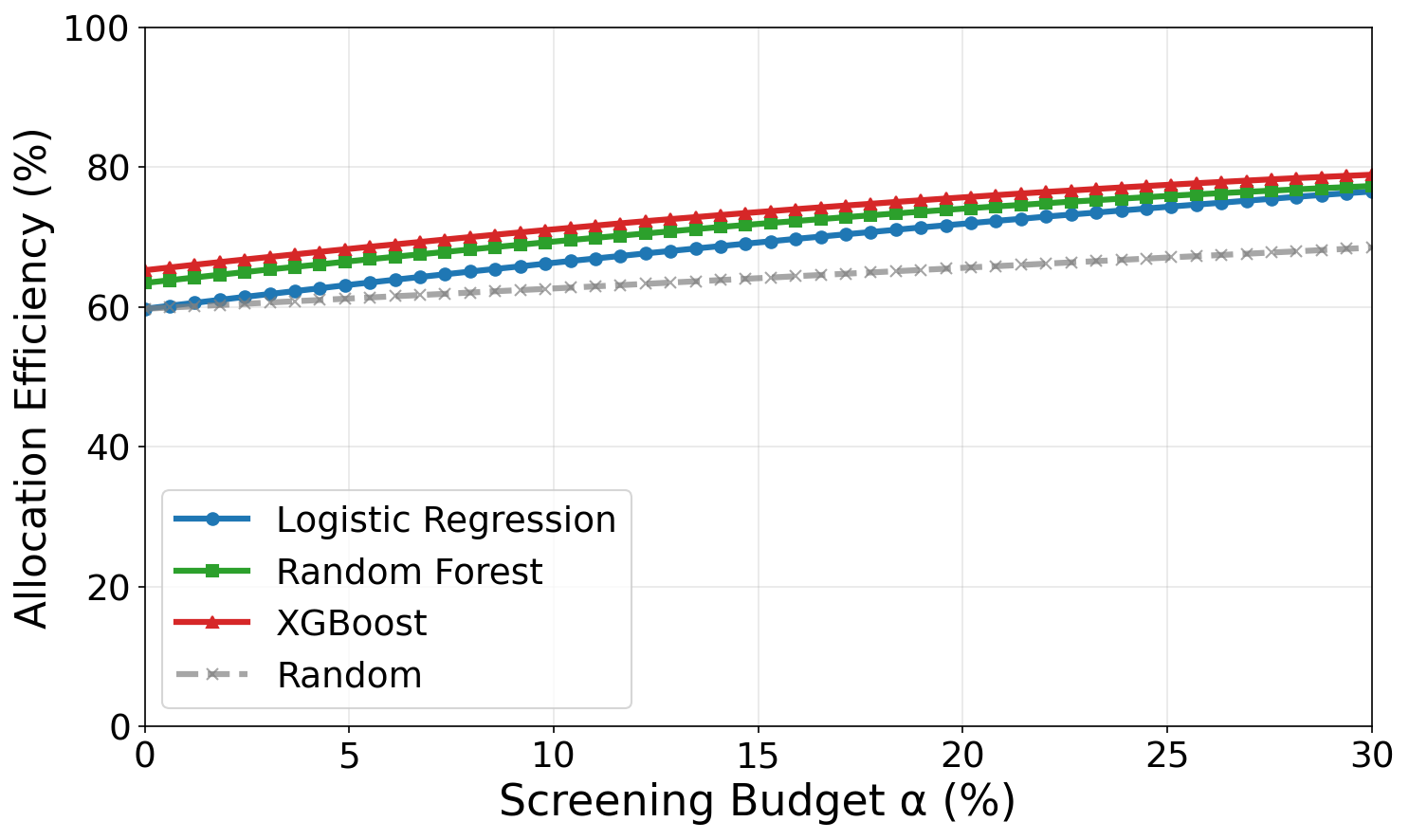}
    \end{minipage}
    \caption{Precision (True Positives / $\beta$) as a function of the screening budget $\alpha$ under optimal screening allocation. Left: landmine contamination prediction in Colombia ($\beta = 10\%$). Right: ACS 2018 low-income prediction ($\beta = 30\%$). Across all models, optimal screening consistently improves allocation precision over both the no-screening baseline ($\alpha = 0$) and random screening.}
    \label{fig:screening_real_world}
\end{figure}

\paragraph{Results.} Figure \ref{fig:screening_real_world} shows the percentage of landmines detected and the percentage of individuals correctly identified as low-income by the two-stage allocation procedure as a function of $\alpha$. Our results demonstrate that optimal screening consistently improves allocation efficiency in real-world allocation settings. In particular, XGBoost achieves the largest gains in both experiments, raising precision from $39\%$ to $70\%$ and from $65\%$ to $80\%$, respectively, at $\alpha = \beta$. Nevertheless, the monotone improvement with $\alpha$ is consistent across models regardless of their predictive performance, confirming that the framework is model-agnostic. Crucially, random screening yields small gains ($35\%$ to $39\%$ and $60\%$ to $65\%$, respectively), while optimal screening nearly doubles precision for the best-performing model. The comparison against random screening shows that the benefits of screening are not simply a function of the budget $\alpha$, but of the screening criteria used. Unlike the synthetic experiments, here the risk scores are estimated from data, so results reflect both aleatoric uncertainty and estimation error. Hence, our experiments suggest that our framework can still deliver substantial gains in net benefit when using realistically estimated (rather than Bayes-optimal) risk scores. \looseness=-1

\section{Discussion and Limitations}

Our work characterizes the fundamental limits of AI-driven allocation under aleatoric uncertainty and provides a principled framework for optimally combining algorithmic targeting and screening. 
We provide theoretical and empirical evidence that these two instruments are complements rather than substitutes: algorithmic targeting identifies the highest-risk units for direct allocation, while screening resolves aleatoric uncertainty for uncertain candidates. 
Operating under realistic budgets, our results identify where screening is most valuable and quantify the efficiency gains it yields. The concavity of the value of screening provides decision-makers with a principled basis for calibrating screening budgets, justifying targeted rather than universal screening, and determining when and where human verification is worth its cost. Taken together, these findings imply that improving algorithmic targeting should not come at the expense of eliminating screening, and that realizing the full potential of prediction-based allocation requires their principled integration.

Our framework is most useful when four conditions hold: (i) screening is cheaper per unit than allocation, $c_s < c_a$ (Corollary~\ref{cor:cost_ratio}); (ii) universal screening is infeasible or prohibitively expensive; (iii) a risk model already exists or is feasible to obtain; and (iv) aleatoric uncertainty near the allocation cutoff is non-trivial, since the gains from screening grow with it, as we demonstrated empirically.

While our work assumes access to Bayes-optimal risk estimates, the practical value of screening is sensitive to estimation quality, as shown in our real-world applications. Epistemic uncertainty in the risk model may shift the screening set away from the optimal band, and understanding how estimation error propagates into screening decisions is a promising direction for future work. Extending the framework to imperfect screening, where noisy outcomes are revealed, is also a valuable research avenue. \looseness=-1

\clearpage
\bibliographystyle{plain}
\bibliography{references}

\clearpage
\appendix

\section{Notation}
\label{app:notation}
\begin{table}[h]
\centering
\begin{tabular}{ll}
\toprule
\textbf{Symbol} & \textbf{Definition} \\
\midrule
$X \in \mathcal{X}$ & Covariates and feature space. \\
$Y \in \{0,1\}$ & Vulnerability status indicating benefit from resource. \\
$\mu(X) = \mathbb{P}(Y=1 \mid X)$ & Bayes-optimal risk score. \\
$\beta \in (0,1)$ & Fixed allocation budget as fraction of population. \\
$\alpha \in (0,\beta)$ & Screening budget as fraction of population. \\
$S: \mathcal{X} \to \{0,1\}$ & Screening rule. \\
$T: \mathcal{X} \to \{0,1\}$ & Allocation rule for unscreened units based on covariates. \\
$\pi(X, Y)$ & Two-stage allocation rule. \\
$\widetilde{q}_\beta$ & $1-\beta$ quantile of $\mu(X)$; purely AI allocation threshold. \\
$q_\alpha, q_\beta$ & Lower and upper boundaries of the optimal screening band. \\
$[q_\alpha, q_\beta]$ & Optimal screening band. \\
$V(\alpha)$ & Allocation value with screening budget $\alpha$ (for fixed $\beta$). \\
$V^{*}(\alpha)$ & Allocation value under optimal $\alpha$-screening rule (for fixed $\beta$). \\
$\mathbb{P}_{X,Y}$ & Joint distribution of covariates and outcome. \\
$F$ & Cumulative distribution function (CDF) of $\mu$. \\ 
\bottomrule
\end{tabular}
\caption{Summary of notation.}
\label{tab:notation}
\end{table}

\section{Proofs}
\label{app:proofs}

\begin{proof}[Proof of Proposition~\ref{prop:optimal_no_screening} (Optimal Allocation Rule without Screening)]
We proceed by contradiction. Let $q^* = q_{1-\beta}$ denote the $(1-\beta)$-quantile of $\mu(X)$, and suppose $T$ is an optimal allocation rule that differs from the threshold rule $T^*(X) = \mathbf{1}\{\mu(X) > q^*\}$. Since $T \neq T^*$ and both exhaust exactly budget $\beta$, they must differ on a set of positive measure. In particular, there exists a unit $x_0$ allocated by $T^*$ but not by $T$, i.e., $\mu(x_0) > q^*$ and $T(x_0) = 0$. Conversely, since both rules exhaust the same budget $\beta$ but $T$ does not allocate $x_0$, it must allocate some unit $x_1$ that $T^*$ does not, i.e., $\mu(x_1) \leq q^*$ and $T(x_1) = 1$. In particular it follows that $\mu(x_0) > q^* \geq \mu(x_1)$, so $\mu(x_0) > \mu(x_1)$.

Define $\tilde{T}$ to agree with $T$ everywhere except at $x_0$ and $x_1$, specifically:
$$\tilde{T}(x_0) = 1, \qquad \tilde{T}(x_1) = 0.$$
This swap preserves feasibility: $\mathbb{E}[\tilde{T}(X)] = \mathbb{E}[T(X)] = \beta$. The change in objective is:
$$\mathbb{E}[\tilde{T}(X)\,\mu(X)] - \mathbb{E}[T(X)\,\mu(X)] = \mu(x_0) - \mu(x_1) > 0,$$
strictly improving the value, contradicting the optimality of $T$. Hence no such $T$ can exist, and the optimal rule must be the quantile threshold $T^*(X) = \mathbf{1}\{\mu(X) > q^*\}$.
\end{proof}

\begin{proof}[Proof of Theorem~\ref{thm:optimal_screening} (Optimal Screening in Two-Stage Allocation)]

To prove the result, we first show that the optimal screening set $S^*(X)$ lies entirely below the optimal direct allocation set $T^*(X)$. With this, the second-stage allocation for unscreened units reduces to a no-screening optimization problem with a threshold rule solution. Finally, we show that the two optimal sets are contiguous, with the screening band sitting immediately below the direct allocation band, and that the boundaries are uniquely determined by the budget constraints.

\textit{Step 1: No directly allocated unit lies below a screened one.} Suppose for contradiction that under the optimal policy, unit $z$ with $\mu(z) < \mu(z')$ is directly allocated while unit $z'$ is screened. Consider the alternative policy that swaps their roles: screen $z$ and directly allocate $z'$. This swap preserves the screening budget constraint since $\mathbb{E}[S(X)]$ is unchanged. The direct contribution to the objective from these two units is unchanged since both screened and directly allocated units contribute $\mu(X)$ in expectation. However, the allocation budget used by the screened unit changes from $\mu(z')$ to $\mu(z) < \mu(z')$, saving $\mu(z') - \mu(z) > 0$ of allocation budget. This extra budget can be used to extend the direct allocation band to additional units with positive risk, strictly increasing the objective. This contradicts optimality of the original policy. Hence at optimality, no directly allocated unit has lower risk than any screened unit, i.e., the screening set lies entirely below the direct allocation set.

\textit{Step 2: Given $S$, the optimal allocation rule is a threshold in $\mu(X)$.} Fix any feasible screening rule $S$ with $\mathbb{E}[S(X)] \leq \alpha$. For screened units, the allocation rule is fixed at $\pi(X,Y) = Y$ by definition. For unscreened units, the remaining allocation problem is
\begin{equation}
    \max_{T: \mathcal{X} \to \{0,1\}} \mathbb{E}[(1-S(X))T(X)\mu(X)] \quad \text{s.t.} \quad \mathbb{E}[S(X)Y] + \mathbb{E}[(1-S(X))T(X)] \leq \beta.
\end{equation}
This is identical in structure to the no-screening problem of Proposition~\ref{prop:optimal_no_screening} with effective budget $\beta - \mathbb{E}[S(X)Y]$ on the subpopulation with $S(X) = 0$. By Proposition~\ref{prop:optimal_no_screening}, the optimal unscreened allocation rule is a threshold rule in $\mu(X)$ within this subpopulation. From Step 1, we have that the screening set lies entirely below the direct allocation set, and the subpopulation threshold translates directly to a threshold in the original risk score space, $T^*(X) = \mathbf{1}\{\mu(X) > q_\beta\}$ for some $q_\beta$.

\textit{Step 3: The optimal screening interval is $[q_\alpha, q_\beta]$.} By Steps 1 and 2, and Proposition \ref{prop:convexity} the optimal policy has the following form: screen some interval $[q_a, q_b]$ with $q_b \leq q_\beta$ and directly allocate $[q_\beta, 1]$, where $q_\beta$ is determined by the allocation budget constraint. We now show the optimal screening interval satisfies $q_b = q_\beta$.

Consider the objective as a function of $q_b$, with $q_a = F^{-1}(F(q_b) - \alpha)$ determined by the screening budget constraint, and $q_\beta = q_\beta(q_b)$ determined by the allocation budget constraint:
\begin{equation}
    \int_{q_a}^{q_b} \mu \, dF(\mu) + 1 - F(q_\beta(q_b)) = \beta.
\end{equation}
The objective is:
\begin{equation}
    V(q_b) = \int_{q_a}^{q_b} \mu \, dF(\mu) + \int_{q_\beta(q_b)}^{1} \mu \, dF(\mu).
\end{equation}
Differentiating with respect to $q_b$, using $\frac{dq_a}{dq_b} = \frac{f(q_b)}{f(q_a)}$ from the screening budget constraint and $\frac{dq_\beta}{dq_b} = \frac{f(q_b)(q_b - q_a)}{f(q_\beta)}$ from the allocation budget constraint:
\begin{equation}
    \frac{dV}{dq_b} = q_b f(q_b) - q_a f(q_b) - q_\beta f(q_b)(q_b - q_a) = f(q_b)(q_b - q_a)(1 - q_\beta) > 0,
\end{equation}
since $q_b > q_a$ and $q_\beta < 1$, for $\alpha, \beta > 0$, and $f > 0$ from having full support on $[0,1]$. Hence $V$ is strictly increasing in $q_b$ and the optimal screening interval is pushed as high as possible. An alternative approach is presented in the proof of Proposition \ref{prop:convexity} which directly shows that the gains from screening higher-risk units are larger than the reduction in direct allocation value from a smaller available budget.

The upper boundary is therefore $q_b = q_\beta$, and $q_\alpha = F^{-1}(F(q_\beta) - \alpha)$ from the screening budget constraint.
\end{proof}

\begin{proposition}[Interval Structure of Optimal Screening]
\label{prop:convexity}
Let $\alpha \in (0,\beta)$ and $\beta \in (0,1)$ with $\alpha + \beta < 1$. The optimal screening rule $S^*$ has an interval structure in the risk score space of the form $S^*(X) = \mathbbm{1}\{\mu(X) \in [q_a, q_b]\}$ for some quantiles $0 < q_a \leq q_b < 1$.
\end{proposition}

\begin{proof}[Proof of Proposition~\ref{prop:convexity} (Interval Structure of Optimal Screening)]
Suppose by contradiction that there is a gap in the optimal screening set $S^*$. That is, there exist units $x_a, x_b$ with $\mu(x_a) < \mu(x_b)$, both screened, and a unit $x_m$ with $\mu(x_a) < \mu(x_m) < \mu(x_b)$ that is \emph{not} screened. Note that $x_m$ is also not directly allocated from Step 1 in the proof of Theorem \ref{thm:optimal_screening}.

Consider swapping $x_a$ out of the screening set and replacing it with $x_m$, keeping all else equal, and define this rule as $\widetilde{S}$, satisfying the screening budget $\EE[S^*] = \EE[\widetilde{S}] = \alpha$. Since $\mu(x_m) > \mu(x_a)$, $\widetilde{S}$ now covers a unit with strictly higher risk, yielding strictly more expected positives from screening:
$$\mathbb{E}[\widetilde{S}Y] - \mathbb{E}[S^* Y] = \mu(x_m) - \mu(x_a) > 0.$$
However, the expected saved budget decreases by $\mu(x_m) - \mu(x_a)$, reducing the budget available for direct allocation. Define $q^* = q_{1-\beta+\EE[S^*Y]}$ and $\widetilde{q} = q_{1-\beta+\EE[\widetilde{S}Y]}$ as the optimal threshold rules for the direct allocation set from Proposition \ref{prop:optimal_no_screening}, from screening sets $S^*$ and $\widetilde{S}$, respectively. We have that the change in contribution from direct allocation is then
$$\int_{q^*}^1\mu dF(\mu) - \int_{\widetilde{q}}^1\mu dF(\mu) = \int_{q^*}^{\widetilde{q}}\mu dF(\mu) \leq \widetilde{q}(\mu(x_m) - \mu(x_a)) < (\mu(x_m) - \mu(x_a)),$$
where the first inequality follows from $F(\widetilde{q}) - F(q^*) = \EE[\widetilde{S}Y] - \EE[S^*Y] = \mu(x_m) - \mu(x_a)$, and the second from $\widetilde{q} < 1$ provided $\alpha + \beta < 1$. 

Therefore, the gains from screening higher-risk units exceed the reduction in contribution from direct allocation given a smaller available budget. That is, the net change in allocation efficiency is strictly positive, contradicting the optimality of $S^*$. Hence the optimal screening set is an interval in the risk space.

\end{proof}

\begin{proof}[Proof of Theorem~\ref{thm:value_screening} (Marginal Value of Screening)]
The proof proceeds in two steps: we first differentiate $V^*(\alpha)$ via Leibniz's rule, and then obtain $dq_{\beta}/d\alpha$ via implicit differentiation of the budget constraint.

By Leibniz's rule applied to the allocation value under optimal screening \eqref{eq:value}:
$$\frac{dV^*}{d\alpha} = -q_{\alpha} f(q_\alpha) \frac{dq_{\alpha}}{d\alpha} = -q_{\alpha} f(q_\alpha)\frac{1}{f(q_{\alpha})}\left(f(q_{\beta})\frac{dq_{\beta}}{d\alpha} - 1\right),$$
where $\frac{dq_{\alpha}}{d\alpha}$ is obtained from the relationship $q_{\alpha} = F^{-1}(F(q_{\beta}(\alpha)) - \alpha)$. Moreover, from this expression we have the partial derivatives for $q_\alpha$:
$$\frac{\partial q_{\alpha}}{\partial q_{\beta}} = \frac{f(q_\beta)}{f(q_{\alpha})}, \qquad
\frac{\partial q_{\alpha}}{\partial \alpha} = -\frac{1}{f(q_{\alpha})}.$$

To derive $\frac{dq_{\beta}}{d\alpha}$, define the budget function
$$ G(q_{\beta}, \alpha) \equiv \int_{q_{\alpha}}^{q_{\beta}} \mu dF(\mu) + 1 - F(q_{\beta}) - \beta =0.$$
Applying Leibniz's rule to differentiate $G$ with respect to $q_{\beta}$:
$$\frac{\partial G}{\partial q_{\beta}} = q_{\beta} f(q_{\beta}) - q_{\alpha} f(q_{\alpha}) \frac{\partial q_{\alpha}}{\partial q_{\beta}} - f(q_{\beta}) = f(q_{\beta})\bigl(q_{\beta} - q_{\alpha} - 1\bigr).$$
Differentiating $G$ with respect to $\alpha$:
$$\frac{\partial G}{\partial \alpha} = -q_{\alpha} f(q_{\alpha}) \frac{\partial q_{\alpha}}{\partial \alpha} = q_{\alpha}.$$
By the implicit function theorem:
$$\frac{dq_{\beta}}{d\alpha} = -\frac{\partial G / \partial \alpha}{\partial G / \partial q_{\beta}} = -\frac{q_{\alpha}}{f(q_{\beta})(q_{\beta} - q_{\alpha} - 1)} = \frac{q_{\alpha}}{f(q_{\beta})(1 - q_{\beta} + q_{\alpha})}.$$
This derivative is always positive as $1 > q_{\beta} \geq  q_{\alpha} > 0$. %That is, as budget increases the upper limit of the screening interval also increases (and the lower limit decreases from its partial derivative being always negative).

Finally, replacing this result for $dq_{\beta}/d\alpha$ in $dV^*/d\alpha$:
$$\frac{dV^*}{d\alpha} = -q_{\alpha} {f(q_{\alpha})} \frac{1}{f(q_{\alpha})}\left(f(q_{\beta})\frac{dq_{\beta}}{d\alpha} - 1\right)= -q_{\alpha}\left(\frac{q_{\alpha}}{1 - q_{\beta} + q_{\alpha}} - 1\right) = \frac{q_{\alpha}(1 - q_{\beta})}{1 - q_{\beta} + q_{\alpha}}.$$
Positivity follows immediately again by $1 - q_{\beta} + q_{\alpha} > 0$, from $q_\alpha, q_\beta \in (0,1)$.

To establish concavity, we differentiate the expression for $dV^*/d\alpha$ with respect to $\alpha$. Substituting the expressions for $dq_\alpha/d\alpha$
and $dq_\beta/d\alpha$ established above:
$$\frac{d^2V^*}{d\alpha^2} = -\frac{(1-q_\beta)^3 f(q_\beta) + q_\alpha^3 f(q_\alpha)}{f(q_\alpha)f(q_\beta)(1-q_{\beta} + q_{\alpha})^3} < 0,$$
where the inequality follows since $q_\alpha, q_\beta \in (0,1)$. Hence $V^*(\alpha)$ is concave in $\alpha$, formalizing a law of diminishing returns to screening.

\end{proof}

\begin{proof}[Proof of Corollary~\ref{cor:at_the_margin}]

Following the proof of Theorem \ref{thm:value_screening}, we have that
$$\frac{dq_{\beta}}{d\alpha} = \frac{q_{\alpha}}{f(q_{\beta})(1 - q_{\beta} + q_{\alpha})},$$
which is always positive given $0 < q_{\alpha} \leq  q_{\beta} < 1$. That is, as the screening budget increases, the upper limit of the screening interval also increases.

Moreover, from the budget constraint $q_{\alpha} = F^{-1}(F(q_{\beta}(\alpha)) - \alpha)$ we have
$$\frac{dq_{\alpha}}{d\alpha} = \frac{1}{f(q_{\alpha})}\left(f(q_{\beta})\frac{dq_{\beta}}{d\alpha} - 1\right) = \frac{1}{f(q_{\alpha})}\left(\frac{q_{\alpha}}{1 - q_{\beta} + q_{\alpha}} - 1\right) = \frac{1}{f(q_{\alpha})}\left(\frac{q_{\beta} - 1}{1 - q_{\beta} + q_{\alpha}}\right),$$
which is always negative for $q_{\beta} < 1$. That is, as the screening budget grows, the lower limit of the screening interval decreases.

To establish that the purely algorithmic allocation threshold $\widetilde{q}_{\beta}$ falls within the optimal screening band, note that when $\alpha = 0$ the screening set collapses to $q_{\alpha} = q_{\beta} = \widetilde{q}_{\beta}$. Then, for $\alpha > 0$, $q_{\alpha}$ decreases and $q_{\beta}$ increases, so  $\widetilde{q}_{\beta} \in [q_{\alpha},  q_{\beta}]$ and the screening set selects units at the margin of algorithmic allocation around $\widetilde{q}_{\beta}$.
\end{proof}

\begin{corollary}
\label{cor:uniform}
    Assume the risk is uniformly distributed, $\mu \sim \text{Unif}[0, 1]$. Then the optimal screening thresholds admit the closed-form expressions:
    $$q_\beta = \frac{1-\beta-\alpha^2/2}{1-\alpha}, \qquad q_\alpha = q_\beta - \alpha.$$
\end{corollary}

\begin{proof}
Under uniform risk $\mu \sim \text{Unif}[0, 1]$, the screening constraint implies
$$F(q_{\alpha}) = F(q_{\beta}) - \alpha \quad\Longrightarrow\quad q_{\alpha} = q_{\beta} - \alpha.$$

Therefore, the allocation constraint \eqref{eq:budget_constraints} becomes 
\begin{align*}
    \beta = \int_{q_{\alpha}}^{q_{\beta}} \mu dF(\mu) + 1 - F(q_{\beta}) = \frac{1}{2}q_{\beta}^2 - \frac{1}{2}(q_{\beta} - \alpha)^2 + 1 - q_{\beta}.
\end{align*}
Solving for $q_{\beta}$ yields the result.
\end{proof}

\begin{proof}[Proof of Proposition ~\ref{cor:at_the_margin}]
Fix $\alpha$ and differentiate the budget constraints \eqref{eq:screening_budget_constraint}--\eqref{eq:allocation_budget_constraint} with respect to $\beta$. Writing $f$ for the density of $F$, the screening constraint gives $f(q_\beta)\,dq_\beta = f(q_\alpha)\,dq_\alpha =: A$, and the allocation constraint gives $q_\beta A - q_\alpha A - A = d\beta$, so that $A = -d\beta / (1 - q_\beta + q_\alpha)$. Since $V^*(\alpha, \beta) = \int_{q_\alpha}^{1} \mu\,dF(\mu)$ (Eq.~\eqref{eq:value}), $dV^* = -q_\alpha f(q_\alpha)\,dq_\alpha = -q_\alpha A$, which yields the expression for $dV^*/d\beta$. Substituting it and $dV^*/d\alpha = q_\alpha(1 - q_\beta)/(1 - q_\beta + q_\alpha)$ (Theorem~\ref{thm:value_screening}) into \eqref{eq:cost_foc}, the common factor $q_\alpha / (1 - q_\beta + q_\alpha)$ cancels and we obtain $1 - q_\beta^* = c_s / c_a$.
\end{proof}

\begin{proof}[Proof of Theorem~\ref{thm:mrs} (Marginal Rate of Substitution between Screening and Allocation)]
The proof follows the same strategy as Theorem~\ref{thm:value_screening}, now differentiating with respect to $\beta$ while holding $\alpha$ fixed. We first differentiate $V^*$ via Leibniz's rule, then obtain $dq_\beta/d\beta$ via implicit differentiation of the budget constraint, and finally take the ratio with $dV^*/d\alpha$.

By Leibniz's rule applied to the allocation value under optimal screening \eqref{eq:value}:
$$\frac{dV^*}{d\beta} = -q_{\alpha} f(q_\alpha) \frac{dq_{\alpha}}{d\beta} = -q_{\alpha} f(q_\alpha) \frac{f(q_{\beta})}{f(q_{\alpha})}\frac{dq_{\beta}}{d\beta} = -q_\alpha f(q_\beta) \frac{dq_{\beta}}{d\beta},$$
where, since $\alpha$ is fixed, $q_\alpha$ depends on $\beta$ only through $q_\beta$, with $\partial q_{\alpha} / \partial q_{\beta} = f(q_\beta)/f(q_{\alpha})$ as in the proof of Theorem~\ref{thm:value_screening}.

To derive $\frac{dq_{\beta}}{d\beta}$, consider the budget function $G$ from the proof of Theorem~\ref{thm:value_screening}, now as a function of $(q_\beta, \beta)$ for fixed $\alpha$:
$$ G(q_{\beta}, \beta) \equiv \int_{q_{\alpha}}^{q_{\beta}} \mu dF(\mu) + 1 - F(q_{\beta}) - \beta = 0.$$
The partial derivative with respect to $q_\beta$ is unchanged,
$$\frac{\partial G}{\partial q_{\beta}} = f(q_{\beta})\bigl(q_{\beta} - q_{\alpha} - 1\bigr),$$
while differentiating with respect to $\beta$ gives $\partial G / \partial \beta = -1$. By the implicit function theorem:
$$\frac{dq_{\beta}}{d\beta} = -\frac{\partial G / \partial \beta}{\partial G / \partial q_{\beta}} = \frac{1}{f(q_{\beta})(q_{\beta} - q_{\alpha} - 1)} = -\frac{1}{f(q_{\beta})(1 - q_{\beta} + q_{\alpha})}.$$
This derivative is always negative as $1 > q_{\beta} \geq q_{\alpha} > 0$. %That is, as the allocation budget increases, the direct allocation threshold decreases (and so does the lower limit of the screening interval).

Replacing this result in $dV^*/d\beta$:
$$\frac{dV^*}{d\beta} = -q_\alpha f(q_\beta) \left(-\frac{1}{f(q_{\beta})(1 - q_{\beta} + q_{\alpha})}\right) = \frac{q_{\alpha}}{1 - q_{\beta} + q_{\alpha}},$$
which is positive since $q_\alpha > 0$ and $1 - q_{\beta} + q_{\alpha} > 0$.

Finally, dividing the expression for $dV^*/d\alpha$ in Theorem~\ref{thm:value_screening} by $dV^*/d\beta$:
$$\frac{dV^*/d\alpha}{dV^*/d\beta} = \frac{q_{\alpha}(1 - q_{\beta})}{1 - q_{\beta} + q_{\alpha}} \cdot \frac{1 - q_{\beta} + q_{\alpha}}{q_{\alpha}} = 1 - q_\beta.$$
\end{proof}

\begin{proof}[Proof of Corollary~\ref{cor:cost_ratio}]
Since $V^*(\alpha, \beta)$ is increasing in both arguments (Theorems \ref{thm:value_screening} and \ref{thm:mrs}), the budget constraint binds, and we can write $\beta(\alpha) = (C - c_s \alpha)/c_a$. Along this line, by Theorem~\ref{thm:mrs},
$$\frac{dV^*(\alpha, \beta(\alpha))}{d\alpha} = \frac{dV^*}{d\alpha} - \frac{c_s}{c_a}\frac{dV^*}{d\beta} = \frac{dV^*}{d\beta}\left(1 - q_\beta - \frac{c_s}{c_a}\right),$$
whose sign is that of $1 - q_\beta - c_s/c_a$, since $dV^*/d\beta > 0$. Moreover, using the derivatives of $q_\beta$ from the proofs of Theorems~\ref{thm:value_screening} and~\ref{thm:mrs},
$$\frac{dq_\beta(\alpha, \beta(\alpha))}{d\alpha} = \frac{\partial q_\beta}{\partial \alpha} - \frac{c_s}{c_a}\frac{\partial q_\beta}{\partial \beta} = \frac{q_\alpha + c_s/c_a}{f(q_\beta)(1 - q_\beta + q_\alpha)} > 0,$$
so $q_\beta$ is strictly increasing along the budget line, and the objective is increasing while $q_\beta < 1 - c_s/c_a$ and decreasing afterwards. Hence any interior maximum satisfies $q_\beta^* = 1 - c_s/c_a$, and it is unique.

At $\alpha = 0$, the entire budget is allocated directly, so $1 - F(q_\beta) = C/c_a$. If $C \leq c_a(1 - F(1 - c_s/c_a))$, then $q_\beta \geq 1 - c_s/c_a$ already at $\alpha = 0$, the objective is non-increasing along the whole budget line, and the optimum is $\alpha^* = 0$ and $\beta^* = C/c_a$. Otherwise, $q_\beta < 1 - c_s/c_a$ at $\alpha = 0$, and the maximum is attained where $q_\beta^* = 1 - c_s/c_a$.
\end{proof}

\section{Algorithm for Computing the Optimal Screening Quantiles}

\label{app:algorithm}
\begin{algorithm*}[ht]
  \caption{Fixed-Point Algorithm for Optimal Screening and Allocation Assignment}
  \label{algo}
  \begin{algorithmic}[1]
  \Require Population sorted by increasing risk, allocation budget $\beta$, screening budget $\alpha$.
  \Ensure Band assignments that maximize allocation efficiency
  \State $\rho_{\text{prev}} \gets 0$ \Comment{Initialize average risk of Screening Band}
  \For{$k = 1$ to $\mathtt{max\_iterations}$}
      \State \textbf{Compute target mass for each band:}
      \State $m_{alloc} \gets \beta - \alpha$ \Comment{Direct Allocation Band}
      \State $m_{res} \gets \alpha \cdot (1 - \rho_{\text{prev}})$ \Comment{Residual Allocation Band from saved budget}
      \State $m_{screen} \gets \alpha$ \Comment{Screening Band}
      \State
      \State \textbf{Assign units to bands by cumulative mass:}
      \State $i_{alloc} \gets \min\{i : \sum_{j=i}^{n} \frac{1}{n} \geq m_{alloc}\}$ \Comment{Start of Direct Allocation Band}
      \State $i_{res} \gets \min\{i : \sum_{j=i}^{n} \frac{1}{n} \geq m_{alloc} + m_{res}\}$ \Comment{Start of Residual Allocation Band}
      \State $i_{screen} \gets \min\{i : \sum_{j=i}^{n} \frac{1}{n} \geq m_{alloc} + m_{res} + m_{screen}\}$ \Comment{Start of Screening Band}
      \State
      \State \textbf{Recompute average risk in Screening Band:}
      \State $\rho_{\text{curr}} \gets \frac{1}{i_{res} - i_{screen}} \sum_{j=i_{screen}}^{i_{res}-1} \mu_j$ \Comment{$\mu_j$ = risk of unit $j$}
      \State
      \If{$|\rho_{\text{curr}} - \rho_{\text{prev}}| < \mathtt{tolerance}$}
          \State \textbf{break} \Comment{Convergence achieved}
      \EndIf
      \State $\rho_{\text{prev}} \gets \rho_{\text{curr}}$ \Comment{Update for next iteration}
  \EndFor
  \State
  \Return Band assignments $(i_{alloc}, i_{res}, i_{screen})$ with converged $\rho_{\text{prev}}$
  \end{algorithmic}
\end{algorithm*}

\begin{theorem}[Fixed Point: Existence, Uniqueness, and Convergence of Algorithm \ref{algo}]\label{thm:fixed_point}
    Let $\alpha \in (0, \beta)$ and $\beta + \alpha < 1$. The update function  $g(\rho) = \EE[\mu \mid \mu \in [q_{\alpha}(\rho), q_{\beta}(\rho)]]$ of Algorithm \ref{algo} is a contraction on $[0, 1]$ with nonnegative Lipschitz constant
    $$c_F := F^{-1}(1-\beta+\alpha) - F^{-1}(1-\beta-\alpha) < 1.$$
    By Banach's fixed point theorem, there exists a unique $\rho^* \in [0,1]$ with $g(\rho^*) = \rho^*$, and Algorithm \ref{algo} converges for any $\rho^{(0)} \in [0,1]$:
    $$|\rho^{(k)} - \rho^*| \leq c_F^k\,|\rho^{(0)} - \rho^*|.$$
\end{theorem}

\begin{proof}[Proof of Theorem~\ref{thm:fixed_point}]
    %We will proceed by first proving that the function $g$ is continuous. 
    %Under $\alpha \in (0,\beta)$ and $\beta+\alpha < 1$, both arguments of $F^{-1}$ lie strictly in $(0,1)$ for all $\rho \in [0,1]$, so $g$ is well-defined, continuous, 
    
    Define the update function $g$ as the average risk over the screening band:
    $$g(\rho) = \frac{1}{\alpha}\int_{q_{\alpha}(\rho)}^{q_{\beta}(\rho)} \mu \, dF(\mu),$$
    where the band thresholds
    \begin{align*}
        q_{\beta}(\rho) &= F^{-1}(1 - \beta + \alpha\rho), \\
        q_{\alpha}(\rho) &= F^{-1}(1 - \beta - \alpha + \alpha\rho),
    \end{align*}
    are determined by the budget constraints of Theorem~\ref{thm:optimal_screening} and the average screening risk $\rho$, satisfying $q_\alpha < q_\beta$ and $F(q_\beta) - F(q_\alpha) = \alpha$. 

    First note that since $g(\rho)$ is an average of $\mu \in [0,1]$, we have that $g(\rho) \in [0,1]$ for all $\rho \in [0,1]$, and $g$ maps $[0,1]$ into $[0,1]$. 
    
    We now show that $g$ is a contraction. Let $\rho_1 < \rho_2$, with corresponding screening bands $[q_{\alpha}(\rho_1), q_{\beta}(\rho_1)]$ and $[q_{\alpha}(\rho_2), q_{\beta}(\rho_2)]$ of $F$-mass $\alpha$. Given that $F^{-1}$ is nondecreasing and $\rho_2 - \rho_1 \leq 1$, we have  $q_{\alpha}(\rho_2) = F^{-1}(1 - \beta - \alpha(1-\rho_2)) \leq F^{-1}(1 - \beta + \alpha\rho_1) = q_{\beta}(\rho_1)$. That is, the two intervals always overlap (or touch at the boundary).
     
    The overlapping region $[q_{\alpha}(\rho_2), q_{\beta}(\rho_1)]$ cancels in the difference, yielding 
    $$g(\rho_2) - g(\rho_1) = \frac{1}{\alpha}\left[\int_{q_\beta(\rho_1)}^{q_\beta(\rho_2)}\mu\,dF(\mu) - \int_{q_\alpha(\rho_1)}^{q_\alpha(\rho_2)}\mu\,dF(\mu)\right].$$
    Both integrals are over intervals of $F$-mass $\alpha(\rho_2-\rho_1)$. Bounding by the supremum and infimum of each region:
    \begin{align*}
        \int_{q_\beta(\rho_1)}^{q_\beta(\rho_2)}\mu\,dF(\mu) &\leq q_\beta(\rho_2)\cdot\alpha(\rho_2-\rho_1) \leq F^{-1}(1-\beta+\alpha)\cdot\alpha(\rho_2-\rho_1),\\
        \int_{q_\alpha(\rho_1)}^{q_\alpha(\rho_2)}\mu\,dF(\mu) &\geq q_\alpha(\rho_1)\cdot\alpha(\rho_2-\rho_1) \geq F^{-1}(1-\beta-\alpha)\cdot\alpha(\rho_2-\rho_1),
    \end{align*}
    where $F^{-1}(1-\beta+\alpha)$ and $F^{-1}(1-\beta-\alpha)$ correspond to $q_{\beta}(1)$ and $q_{\alpha}(0)$, respectively. Subtracting and dividing by $\alpha$:
    $$0 \leq g(\rho_2) - g(\rho_1) \leq c_F(\rho_2-\rho_1),$$
    where $c_{F} = F^{-1}(1-\beta+\alpha) - F^{-1}(1-\beta-\alpha).$ Since $\alpha < \beta$ implies $F^{-1}(1-\beta+\alpha) < 1$, and $\beta+\alpha < 1$ implies $F^{-1}(1-\beta-\alpha) > 0$, we have $c_F < 1$. The first inequality $g(\rho_2) - g(\rho_1) \geq 0$ follows because the first integral dominates the second, establishing that $g$ is monotone increasing. 
    
    Hence $|g(\rho_1)-g(\rho_2)| \leq c_F|\rho_1-\rho_2|$ and $g$ is a contraction on the complete metric space $([0,1], |\cdot|)$. The result follows from Banach's fixed point theorem.
\end{proof}

\section{Empirical Analysis on Synthetic Data}
\label{app:simulations}

We consider four distributions governing the vulnerability risk in our synthetic population, depicted in Figure \ref{fig:risk_dists}. All four distributions share the same expected fraction of vulnerable units ($\EE[\mu]= 0.5$), so differences in targeting outcomes are driven purely by distributional shape. This allows us to study the limits of purely AI-driven allocation when no screening is employed, and to characterize the benefits of screening on allocation efficiency as aleatoric uncertainty increases. Specifically, we consider four risk distributions ordered by increasing aleatoric uncertainty. \textbf{Bimodal} ($\text{Beta}(0.1, 0.1)$): Mass concentrated near 0 and 1, with high separation between low- and high-risk units and minimal aleatoric uncertainty. \textbf{Uniform}: Flat over $[0,1]$ representing moderate uncertainty with random risk level. \textbf{Unimodal bell-shaped} ($\text{Beta}(10,10)$): Concentrated around $0.5$, where most units have intermediate risk and the signal for targeting is weak. \textbf{Point mass} ($\delta_{0.5}$): All units have identical risk $\mu = 0.5$, representing the limiting case of maximum aleatoric uncertainty in which no algorithmic targeting is possible.

\begin{figure}[h]
    \centering
    \includegraphics[width=0.6\linewidth]{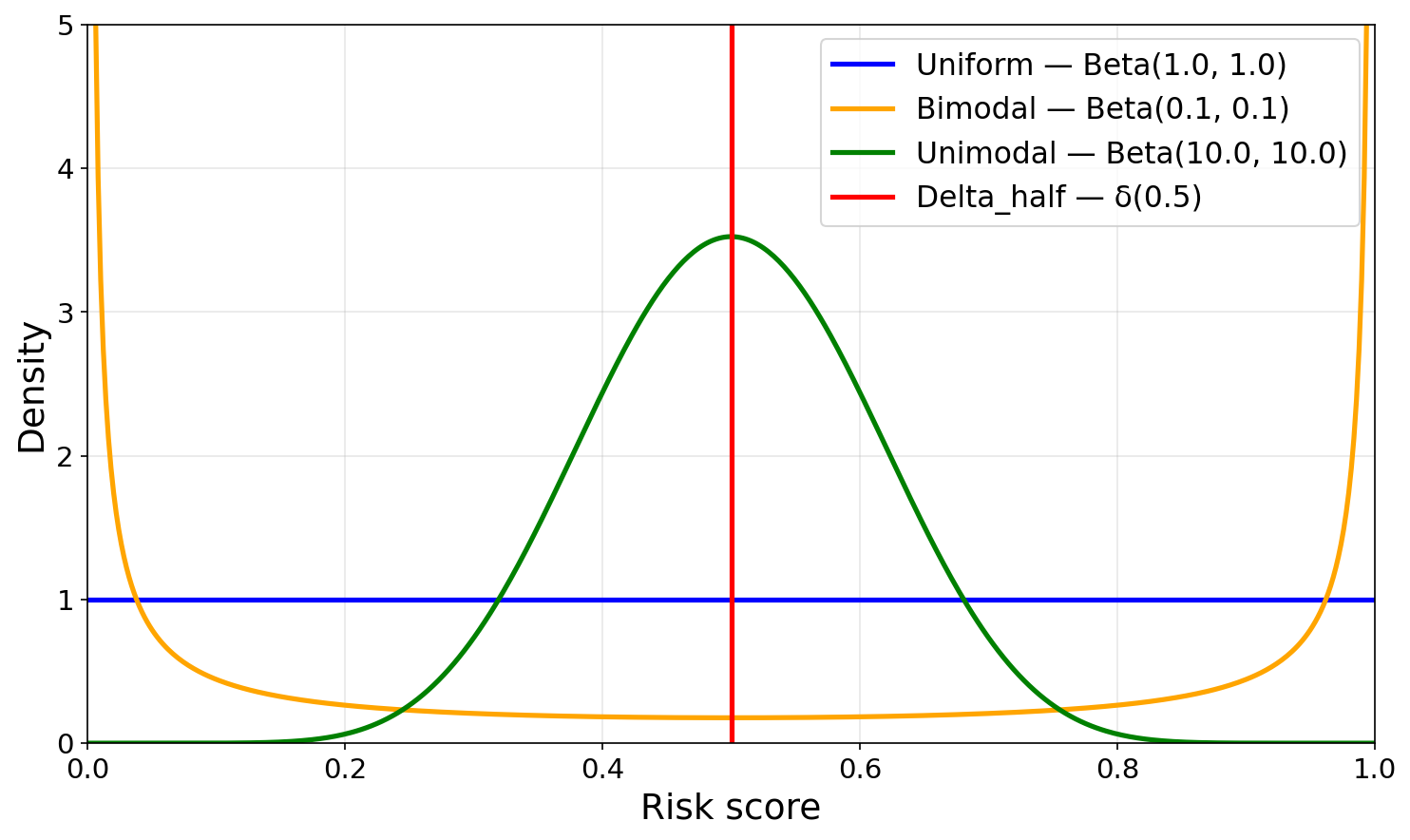}
    \caption{Risk distributions used in simulations. All four share the same mean ($\EE[\mu] = 0.5$) but differ in aleatoric uncertainty, ranging from low (bimodal, with mass concentrated near 0 and 1) to maximum ($\delta_{0.5}$, with all mass at 0.5).}
    \label{fig:risk_dists}
\end{figure}

For each distribution, we sample $N = 100{,}000$ units with true conditional risks and realized outcomes $Y_i\mid \mu_i \sim \text{Bernoulli}(\mu_i)$. We fix the resource budget at $\beta = 35\%$ and run Algorithm \ref{algo} to obtain the optimal screening set for varying screening budgets $\alpha \in [0, \beta]$. We report allocation precision, defined as the fraction of allocated units that are truly vulnerable (True Positives / $\beta$). Since the total number of targeted units is fixed to $\beta = 35\%$, all improvements in allocation efficiency come from better targeting through observed outcomes via screening.\footnote{All computations were performed on a Lenovo laptop equipped with an Intel Core Ultra 9 258H CPU and 32 GB of RAM.}

We compare against the optimal no-screening policy (Proposition~\ref{prop:optimal_no_screening}) in Figure~\ref{fig:sim_noscreening}. We further compare the optimal screening strategy against random screening, which selects units for physical verification uniformly at random (Figure \ref{fig:more_sim}, left), and a heuristic variant that places the screening band immediately below the purely algorithmic allocation threshold (Figure \ref{fig:more_sim} right). These comparisons illustrate that the efficiency gains from screening depend not only on the total amount of screening capacity deployed, but critically on where in the score distribution screening is concentrated.

\begin{figure}[ht]
    \centering
    \includegraphics[width=0.49\textwidth]{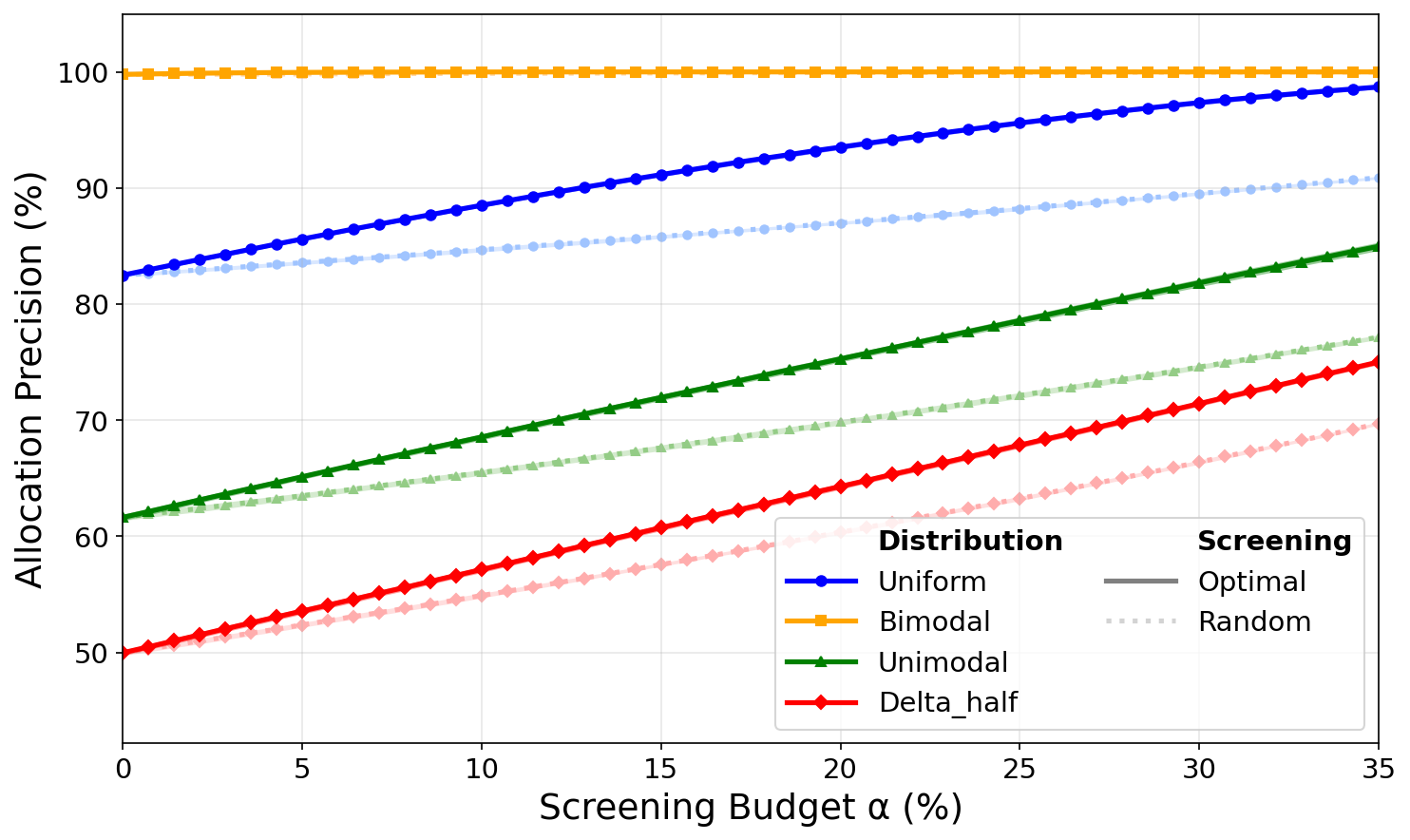}
    %\caption{Allocation precision by screening budget $\alpha$ under optimal screening vs. random screening baseline across risk distributions ($\beta = 35\%$). Lines show averages over 10 simulations; shaded regions indicate $\pm$ one standard deviation.}
\hfill
    \includegraphics[width=0.49\textwidth]{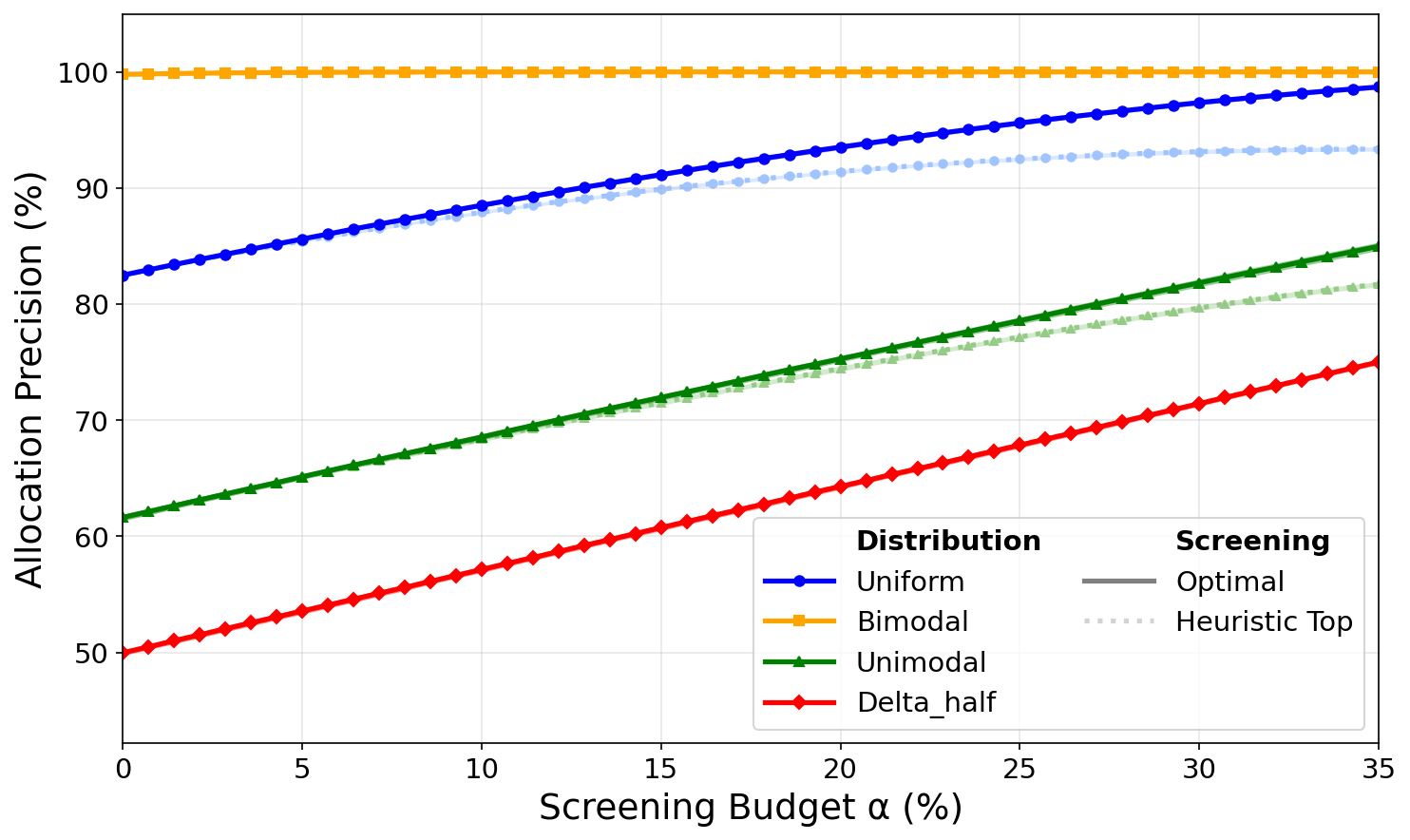}
    \caption{Allocation precision by screening budget $\alpha$ under optimal screening vs. random screening (left) and heuristic top-adjacent screening (right) across risk distributions ($\beta = 35\%$). Lines show averages over 10 simulations; shaded regions indicate $\pm$ one standard deviation, which is negligible at this scale for all $\alpha$.}
    \label{fig:more_sim}
\end{figure}

\section{Comparison with Fixed Heuristic Screening Rules}
\label{app:heuristics}

In practice, a natural way to deploy screening is to apply a simple rule of thumb: verify the units ``near the cutoff,'' where the allocation decision is most uncertain. Such rules are practical because they require neither knowledge of the risk distribution $F$ nor solving for the endpoints of the optimal band. This raises the question of whether computing the optimal policy is worth it, or whether a fixed heuristic performs comparably. In this section, we show that fixed heuristics can lose a substantial fraction of the value of screening, both theoretically, by characterizing their worst-case loss over $F$, and empirically.

We consider fixed screening intervals of the form $[a(u), u]$ parametrized by their upper endpoint $u$, with $F(u) - F(a(u)) = \alpha$. Let
\begin{equation*}
    \rho(u) := \int_{a(u)}^{u} \mu \, dF(\mu)
    \qquad \text{and} \qquad
    \rho^*(\alpha) := \int_{q_\alpha}^{q_\beta} \mu \, dF(\mu)
\end{equation*}
denote the expected number of screened positives under the heuristic and under the optimal band of Theorem~\ref{thm:optimal_screening}, respectively. The residual budget $\beta - \rho(u)$ is allocated directly to the highest-risk unscreened units (Proposition~\ref{prop:optimal_no_screening}), yielding value $V(u)$. Our analysis applies to any such interval, and we instantiate it with two natural heuristics:
\begin{itemize}
    \item \textbf{Top-adjacent:} the screening band sits immediately below the pure algorithmic allocation threshold, $u = \widetilde{q}_\beta$.
    \item \textbf{Symmetric:} the screening band is centered, in quantile terms, at $\widetilde{q}_\beta$, that is, $u = q_s$ with $F(q_s) = F(\widetilde{q}_\beta) + \alpha/2$.
\end{itemize}

First, note that there is no nontrivial distribution-free \emph{lower} bound on the gain of optimal screening over a fixed heuristic. For instance, as $F$ approaches a point mass, any band of mass $\alpha$ selects essentially the same units as the optimal band, so the loss of a heuristic can be made arbitrarily small. We therefore characterize the worst-case loss of any fixed screening interval over risk distributions $F$, and show that the resulting bounds are sharp, including for the two heuristics above.

\begin{proposition}[Worst-Case Loss of Fixed Screening Intervals]
\label{prop:heuristic_gap}
Fix $0 < \alpha < \beta < 1$ and a screening interval $[a(u), u]$ with $F(u) - F(a(u)) = \alpha$.
\begin{enumerate}
    \item[(i)] If $u < q_\beta$, let $r(u)$ solve $F(r(u)) = 1 - \beta + \rho(u)$. Then
    \begin{equation*}
        V^*(\alpha) - V(u) \leq F(r(u)) - F(\max(q_\alpha, u)).
    \end{equation*}
    \item[(ii)] If $u > q_\beta$, then
    \begin{equation*}
        V^*(\alpha) - V(u) \leq \rho(u) - \rho^*(\alpha).
    \end{equation*}
\end{enumerate}
Moreover, both bounds are sharp, in the sense that for each there exists a sequence of distributions $F_n$ along which the loss $V^*(\alpha) - V(u)$ attains the upper bound in the limit.
\end{proposition}

\begin{proof}
\textbf{Case (i), $u < q_\beta$.} Screening lower-risk units yields fewer expected positives, $\rho(u) < \rho^*(\alpha)$, so the direct allocation cutoff $r(u)$ satisfies $u < r(u) < q_\beta$. This leaves a gap $(u, r(u))$ that is neither screened nor allocated. Since $V(u) = \int_{a(u)}^{u} \mu\,dF(\mu) + \int_{r(u)}^{1} \mu\,dF(\mu)$ and $V^*(\alpha) = \int_{q_\alpha}^{1} \mu\,dF(\mu)$ (Eq.~\eqref{eq:value}), we have
\begin{equation*}
    V^*(\alpha) - V(u) = \int_{\max(q_\alpha, u)}^{r(u)} \mu \, dF(\mu) - \int_{a(u)}^{\min(q_\alpha, u)} \mu \, dF(\mu) \leq F(r(u)) - F(\max(q_\alpha, u)),
\end{equation*}
using $\mu \leq 1$ in the first integral and $\mu \geq 0$ in the second. That is, the loss is at most the mass of the uncovered gap. The bound is approached as $F$ concentrates the mass of $[\max(q_\alpha, u), r(u)]$ near $\mu = 1$ and that of $[a(u), \min(q_\alpha, u)]$ near $\mu = 0$.

\textbf{Case (ii), $u > q_\beta$.} Screening higher-risk units yields $\rho(u) \geq \rho^*(\alpha)$ and frees budget $\Delta(u) := \beta - \rho(u) - (1 - F(u)) > 0$, which extends direct allocation \emph{below} the screening window to a threshold $d(u)$ solving $F(a(u)) - F(d(u)) = \Delta(u)$. Coverage is then contiguous, so $V(u) = \int_{d(u)}^{1} \mu\,dF(\mu)$, with $F(d(u)) - F(q_\alpha) = \rho(u) - \rho^*(\alpha) \geq 0$. Hence
\begin{equation*}
    V^*(\alpha) - V(u) = \int_{q_\alpha}^{d(u)} \mu \, dF(\mu) \leq \rho(u) - \rho^*(\alpha),
\end{equation*}
which is approached as $F$ concentrates the mass of $[q_\alpha, d(u)]$ near $\mu = 1$.
\end{proof}

Proposition~\ref{prop:heuristic_gap} holds for any fixed screening interval. Applying it to the two heuristics above gives their worst-case losses.

\paragraph{Top-adjacent heuristic.} Since $q_\alpha \leq \widetilde{q}_\beta \leq q_\beta$ (Corollary~\ref{cor:at_the_margin}), $u = \widetilde{q}_\beta$ falls in case (i) of Proposition~\ref{prop:heuristic_gap} with $\max(q_\alpha, u) = \widetilde{q}_\beta$. Using $F(\widetilde{q}_\beta) = 1 - \beta$ and $F(r(\widetilde{q}_\beta)) = 1 - \beta + \rho(\widetilde{q}_\beta)$, its worst-case loss is
\begin{equation*}
    V^*(\alpha) - V(\widetilde{q}_\beta) \leq \rho(\widetilde{q}_\beta).
\end{equation*}
That is, in the worst case, the positives gained by the top-adjacent heuristic through screening are exactly offset by the uncovered gap it creates, so it performs no better than no screening at all.

\paragraph{Symmetric heuristic.} Since $F(q_\beta) = 1 - \beta + \rho^*(\alpha)$ and $F(q_s) = 1 - \beta + \alpha/2$, we have $q_s \leq q_\beta$ if and only if $\alpha/2 \leq \rho^*(\alpha)$. In case (i), $\max(q_\alpha, q_s) = q_s$ and the bound is $F(r(q_s)) - F(q_s) = \rho(q_s) - \alpha/2$; in case (ii), the bound is $\rho(q_s) - \rho^*(\alpha)$. Combining both cases, its worst-case loss is
\begin{equation*}
    V^*(\alpha) - V(q_s) \leq \rho(q_s) - \min\!\left(\frac{\alpha}{2}, \rho^*(\alpha)\right).
\end{equation*}
That is, if its band sits below the optimal one, it leaves an interval of mass $\rho(q_s) - \alpha/2$ neither screened nor allocated; if it sits above, it consumes $\rho(q_s) - \rho^*(\alpha)$ more allocation budget than the optimal policy by screening higher-risk units.

Proposition~\ref{prop:heuristic_gap} shows that the advantage of the fixed-point algorithm is not an artifact of a loose upper bound, as the worst-case loss of fixed heuristics is attained. For both heuristics, the optimal placement of the band depends on $F$, so no fixed rule can close the gap without knowledge of the risk distribution.

\paragraph{Empirical comparison.}
Table~\ref{tab:heuristics} reports allocation precision at $\alpha = \beta = 35\%$ for the synthetic distributions of Appendix~\ref{app:simulations} and for an \textit{adversarial} distribution that places risk $\mu \approx 0$ on 58\% of units and $\mu \approx 0.8$ on the remaining 42\%. Unlike the other synthetic distributions, the adversarial distribution does not satisfy $\EE[\mu] = 0.5$. The table also reports results for the humanitarian demining application, with XGBoost risk scores and $\alpha = \beta = 10\%$.

\begin{table}[ht]
    \centering
    \caption{Allocation precision (\%) of optimal screening vs.\ fixed heuristics at $\alpha = \beta$ ($35\%$ for synthetic distributions, $10\%$ for demining). Synthetic results report mean $\pm$ one standard deviation over 10 simulations. Bold indicates the best-performing screening rules for each distribution.}
    \label{tab:heuristics}
    \begin{tabular}{lcccc}
        \toprule
        Distribution & No screening & Top-adjacent & Symmetric & Optimal \\
        \midrule
        Uniform & $82.49 \pm 0.17$ & $93.34 \pm 0.17$ & $97.84 \pm 0.07$ & $\mathbf{98.70 \pm 0.19}$ \\
        Bimodal & $99.78 \pm 0.03$ & $\mathbf{99.99 \pm 0.00}$ & $\mathbf{100.00 \pm 0.00}$ & $\mathbf{99.99 \pm 0.02}$ \\
        Unimodal & $61.63 \pm 0.27$ & $81.69 \pm 0.21$ & $\mathbf{84.84 \pm 0.25}$ & $\mathbf{84.97 \pm 0.35}$ \\
        Point mass ($\delta_{0.5}$) & $49.97 \pm 0.24$ & $\mathbf{75.03 \pm 0.13}$ & $\mathbf{74.98 \pm 0.19}$ & $\mathbf{75.01 \pm 0.22}$ \\
        Adversarial & $81.08 \pm 2.29$ & $83.31 \pm 1.54$ & $91.97 \pm 1.63$ & $\mathbf{95.51 \pm 1.51}$ \\
        \midrule
        Demining (XGBoost) & $36.71$ & $57.14$ & $63.29$ & $\mathbf{69.62}$ \\
        \bottomrule
    \end{tabular}
\end{table}

Because the Beta$(t,t)$ family is symmetric around $\mu = 0.5$, the symmetric heuristic tracks the optimum closely on these distributions. The top-adjacent heuristic, nonetheless, shows a non-negligible gap in the realistic regime of moderate aleatoric uncertainty: 5.4 points on the uniform and 3.3 points on the unimodal distribution. On the adversarial distribution, the gap grows to 3.5 points for the symmetric and 12.2 points for the top-adjacent heuristic, the latter performing close to the no-screening baseline, as predicted by Proposition~\ref{prop:heuristic_gap}. On the demining task, where risk scores are estimated from data, the symmetric and top-adjacent heuristics fall 6.3 and 12.5 points short of the optimal policy, respectively. These results confirm that fixed heuristics can leave significant value on the table once risk departs from symmetric, idealized cases, whereas the fixed-point algorithm (Algorithm~\ref{algo}) adapts to how risk is actually distributed near the cutoff and attains $V^*(\alpha, \beta)$ for any $F$. \looseness=-1

\section{Reproducibility}
\label{app:reproducibility}

\paragraph{Datasets.}
\emph{Humanitarian demining.} The dataset contains 26,761 grid cells of $500\,\text{m} \times 500\,\text{m}$ in Antioquia, Colombia, each described by 63 geographic, socio-demographic, and war-remnant features. The outcome is landmine contamination, with 1.6\% of grid cells having a positive label.\footnote{The data is available at \url{https://huggingface.co/datasets/mdulcer10/landmine-detection}.}

\emph{Income-based social programs.} We use the 2018 American Community Survey (ACS) Public Use Microdata Sample, accessed via the \texttt{folktables} benchmark \citep{ding2021retiring}. The dataset contains 1,611,572 individual records from the U.S. Census Bureau's annual survey of socioeconomic conditions. Each record includes ten demographic and labor-market features: age, class of worker, educational attainment, marital status, occupation code, place of birth, relationship to household head, usual hours worked per week, sex, and race. The binary outcome is whether an individual's total personal income exceeds \$50{,}000 per year, which holds for approximately 37\% of individuals.

\paragraph{Computing environment.}
All experiments use Python 3.11, with dependencies managed via \texttt{uv} and a committed \texttt{uv.lock} lockfile that pins the exact resolved versions. Running \texttt{uv sync} reproduces the environment.

\paragraph{Models.}
We use three tabular baselines (logistic regression, random forest, and XGBoost), each implemented as a \texttt{scikit-learn} pipeline with a \texttt{ColumnTransformer} preprocessor that applies \texttt{StandardScaler} to numeric columns and \texttt{OneHotEncoder} to categorical columns.
% TODO: report hyperparameters for each model and the random seeds used.

\paragraph{Data splitting.}
For humanitarian demining, we use a geographic split: the test set consists of all grid cells in the municipality of Granada (773 cells), and the training set of the remaining 14 municipalities. Since Granada is excluded before fitting, the model is evaluated on both unseen labels and an unseen feature distribution.
% TODO: report the positive rate in Granada.
For income-based social programs, we use a random 50/50 train--test split stratified on the outcome, with approximately 806,000 records in each split.

All precision curves, which report true positives recovered as a function of the screening budget $\alpha$, are computed on the test split only, using the fitted model's predicted risk to rank units and the true outcomes to resolve screened units.

\end{document}